%% file: main.tex
\documentclass{article}

\PassOptionsToPackage{numbers,compress}{natbib}
\usepackage[main, final]{neurips_2026}

\usepackage[utf8]{inputenc} 
\usepackage[T1]{fontenc}    
\usepackage{hyperref}       
\usepackage{url}            
\usepackage{booktabs}       
\usepackage{amsfonts}       
\usepackage{nicefrac}       
\usepackage{microtype}      
\usepackage{xcolor}         

\usepackage{amsmath}
\usepackage{mathtools}
\usepackage{hyperref}
\usepackage{amsthm}
\usepackage{wrapfig}
\usepackage{hyperref}
\usepackage{float}
\usepackage{csvsimple}
\usepackage{makecell}
\usepackage{caption}

\newtheorem{theorem}{Theorem}
\newtheorem{lemma}[theorem]{Lemma}

\newtheorem{proposition}[theorem]{Proposition}
\newtheorem{definition}[theorem]{Definition}
\newtheorem{remark}[theorem]{Remark}

\title{Not all uncertainty is alike: \\ volatility, stochasticity, and exploration}

\author{%
Payam Piray\\
        \AND
\normalfont  University of Southern California \\
\texttt{\small piray@usc.edu} 
}

\begin{document}

\maketitle

\input{0_abstract}
\input{1_intro}

\input{2_related}
\input{3_problem}
\input{4_motonicity}
\input{5_index}
\input{6_results}
\input{7_discussion}

\begin{ack}
This work was supported by an NIMH grant (R21MH134217) and a National Science Foundation CAREER Award (2543567) to Payam Piray. The author declares no competing financial interests.
\end{ack}

\newpage
\bibliographystyle{unsrtnat}
{
\small
\bibliography{references_cause}
}

\newpage
\appendix
\input{Appendix_A}
\input{Appendix_B}
\input{Appendix_C}
\input{Appendix_D}
\input{Appendix_E}
\input{Appendix_F}


\clearpage
\input{checklist.tex}

\end{document}

%% file: 0_abstract.tex
\begin{abstract}
Adaptive decision-making in biological and artificial intelligence requires balancing the exploitation of known outcomes with the exploration of uncertain alternatives. Although prior work suggests that uncertainty generally promotes exploration, it has typically treated distinct sources of environmental uncertainty as equivalent. We consider environments with latent reward states that drift over time (volatility) and are observed through noisy outcomes (stochasticity). Both of them increase posterior uncertainty, yet we show they drive optimal exploration in opposite directions: volatility enhances it, stochasticity suppresses it. We establish this asymmetry formally by extending the Gittins index framework to Gaussian state-space bandits with latent dynamics. We further derive Cause-Aware Uncertainty-Sensitive Exploration (CAUSE), a closed-form exploration bonus obtained via control-as-inference that inherits the same monotonicities. CAUSE matches or exceeds both standard exploration strategies and numerically computed index policies in environments with heterogeneous noise structure. Learning and exploration are governed by the same noise-inference asymmetry, and the framework predicts that pathological noise inference misallocates exploration in an axis-specific way, with implications for computational accounts of psychiatric conditions.
\end{abstract}

%% file: 1_intro.tex
\section{Introduction}

Across computational accounts of exploration, a recurring principle is that uncertainty drives behavior: greater uncertainty, more exploration. This logic underlies a broad family of exploration strategies: upper confidence bound (UCB) methods explore by acting optimistically under uncertainty \citep{auer2002finite, lai1985asymptotically}; Thompson sampling explores by sampling from the posterior over rewards \citep{thompson1933likelihood}; information-directed sampling (IDS) balances expected regret against expected information gain \citep{russo2014learning}; and neuroscience work links exploration to neural and behavioral signatures of uncertainty \citep{wilson2014humans, gershman2019uncertainty, daw2006cortical, frank2009prefrontal}. In each case, more uncertainty prescribes more exploration.

We show that this prescription is incomplete. Not all uncertainty is alike, and some forms of uncertainty should \emph{decrease} exploration. Specifically, we consider environments with latent reward states that change over time (volatility) and are observed through noisy outcomes (stochasticity). Both volatility and stochasticity increase posterior uncertainty about rewards, and standard methods such as UCB and Thompson sampling treat them equivalently, prescribing more exploration in both cases. We prove that they should drive exploration in opposite directions: volatility enhances exploration, while stochasticity suppresses it (Figure~\ref{fig:overview}).

\begin{figure}[t]
  \centering
  \includegraphics[width=\linewidth]{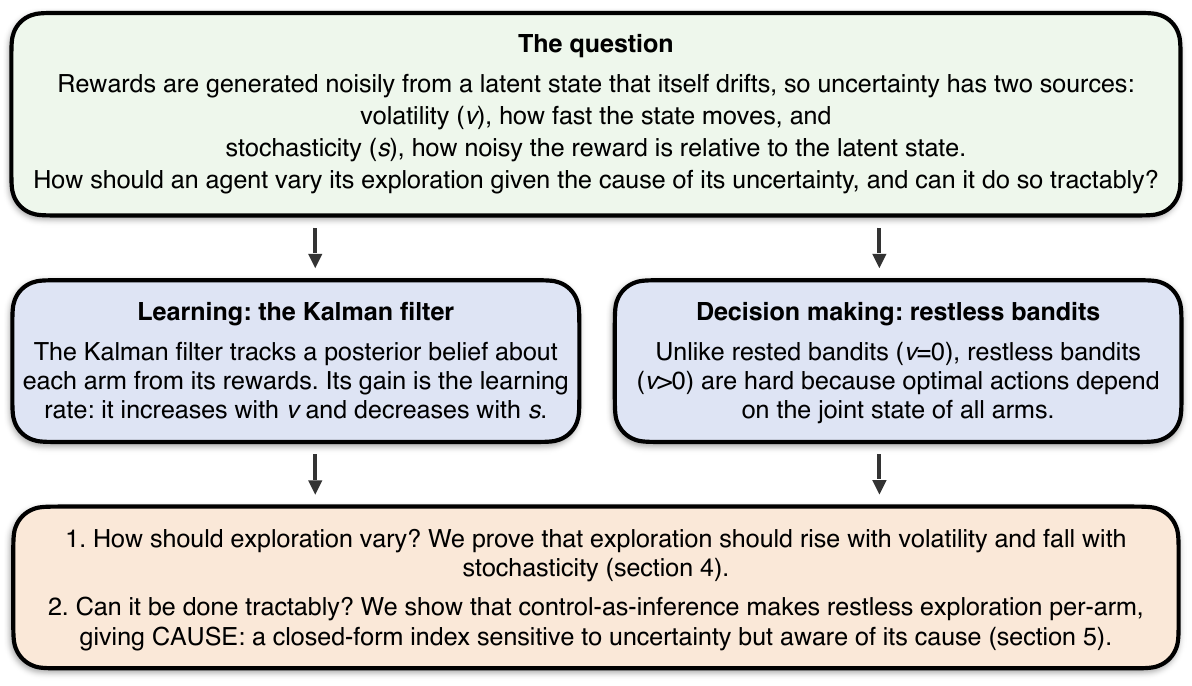}
  \caption{Overview. The distinction between volatility and stochasticity has consequences for both learning and decision making. Previous work has addressed learning \cite{piray2021model, piray2024computational, fang2026inferring}; this work addresses decision making, through the lens of exploration.}
  \label{fig:overview}
\end{figure}

The intuition is that uncertainty matters for exploration only insofar as it reflects potential information gain. Volatility creates information gain: latent states are changing, so new observations reveal something new. Stochasticity destroys it: observations are noisy, so each sample is less informative. An agent that simply explores more whenever it is more uncertain conflates these two cases, exploring vigorously precisely when observations are least useful.

We establish this result formally by extending the Gittins index framework from independent and identically distributed (iid) Gaussian bandits to Gaussian state-space bandits, where each arm has a latent reward state that evolves with volatility and is observed with stochasticity. In this richer setting, we prove that the optimal exploration bonus, defined as the value of information that justifies pulling an arm beyond its expected reward, is monotonically increasing in volatility and decreasing in stochasticity. This result follows from the structure of the Gittins retirement problem and requires no approximations.

The Gittins index is provably optimal only for rested bandits. In restless settings, where unselected arms also evolve, Gittins loses its optimality guarantee and the truly optimal policy is intractable in general. To obtain a practical exploration strategy, we turn to the control-as-inference framework \citep{levine2018reinforcement, todorov2008general}, which casts action selection as posterior inference in a probabilistic graphical model. Within this framework, we derive \emph{Cause-Aware Uncertainty-Sensitive Exploration} (CAUSE), a closed-form index policy for Gaussian state-space bandits. CAUSE decomposes cleanly into exploitation and exploration components, and its exploration component inherits the same qualitative dependence on volatility and stochasticity established by the Gittins analysis. To our knowledge, this is the first application of control-as-inference to derive an exploration index for restless bandits.

These opposing effects parallel a known result in the learning literature. Optimal learning rates increase with volatility and decrease with stochasticity, because the same information-theoretic logic applies: volatility makes new observations more relevant, stochasticity makes them less reliable \citep{piray2021model}. The parallel is not coincidence: both are normative properties of the same generative model.

This unified view yields a testable prediction about pathological behavior. Volatility and stochasticity must be jointly inferred from observations, since both increase experienced noise and can only be disentangled by explaining away. A priori insensitivity to one source leads to systematic misattribution of the other \citep{piray2021model}: an agent with hyposensitive stochasticity priors sees the world as more volatile than it is, and vice versa. In the learning setting, such miscalibration produces \emph{reversed} learning-rate modulation, a pattern observed in human behavior \citep{piray2024computational, fang2026inferring}. We show that the analogous failure in exploration misallocates exploration: stochasticity-blind agents over-explore precisely when observations are least informative, while volatility-blind agents under-explore precisely when the environment is changing fastest. Miscalibrated noise inference has been proposed as a candidate mechanism in several psychiatric conditions, making these misallocations a concrete behavioral target for clinical investigation.

Our contributions are: (i) we extend the Gittins framework from iid Gaussian bandits to Gaussian state-space bandits with latent dynamics, proving that the exploration bonus increases monotonically with volatility and decreases monotonically with stochasticity; (ii) within the control-as-inference framework, we derive CAUSE, a closed-form index policy for Gaussian state-space bandits, which to our knowledge is the first application of control-as-inference to derive such an index; (iii) we show that the opposing effects of volatility and stochasticity on exploration parallel those on learning, both following from the normative properties of the same generative model; (iv) empirically, CAUSE matches or exceeds Thompson sampling, UCB, predictive sampling, IDS, and numerically computed Gittins and Whittle indices across volatility and stochasticity regimes, and is statistically tied with the optimal policy wherever it can be computed; (v) lesioned inference misallocates exploration in an axis-specific way, with the direction diagnostic of which latent inference has failed. Code is available at \url{https://github.com/piraylab/cause-exploration}.

%% file: 2_related.tex
\section{Related work}
\paragraph{Exploration and uncertainty in neuroscience.}
Behavioral and neural studies have linked exploratory behavior to uncertainty signals, including relative and total uncertainty \citep{wilson2014humans}, expected and unexpected uncertainty \citep{cohen2007should}, and information value \citep{daw2006cortical, gershman2018deconstructing}. These accounts generally predict that exploration increases with uncertainty, without distinguishing between the sources of that uncertainty. Our framework predicts that volatility and stochasticity should have opposing effects on exploratory behavior, and that selective failures of inference about these quantities should produce characteristic, axis-specific misallocation rather than uniform deficits. To our knowledge, these qualitative predictions have not been tested experimentally.

\paragraph{Volatility, stochasticity, and adaptive learning.}
Adaptive learning under uncertainty has been modeled using hierarchical Bayesian frameworks that track volatility \citep{behrens2007learning, mathys2011bayesian} or detect change-points \citep{nassar2010approximately, nassar2012rational}. These approaches typically estimate one source of uncertainty while treating the other as known, overlooking the computational challenge of their simultaneous estimation. \citet{piray2021model} introduced a framework for jointly estimating volatility and stochasticity, showing that although both increase outcome variance, they can be distinguished through temporal structure and exert opposing effects on learning rate. The same work showed that selective failures to represent one source produce reversed, rather than merely impaired, learning-rate modulation, offering a candidate account of pathological learning in psychiatric and neurological conditions. This framework was extended to binary outcomes and validated in large-scale behavioral experiments \citep{piray2024computational}. Our work builds on this foundation, showing that the same dissociation drives exploration as a normative result, with analogous lesion-model predictions for exploratory behavior.

\paragraph{Exploration in bandits.}

The explore/exploit dilemma has been studied extensively in the multi-armed bandit literature. UCB methods \citep{auer2002finite, lai1985asymptotically} and Thompson sampling \citep{thompson1933likelihood, daniel2018tutorial} both increase exploration with uncertainty, the former through optimistic reward estimates and the latter through posterior sampling; a Bayesian variant of UCB \citep{kaufmann2012bayesian} replaces frequentist confidence bounds with posterior credible intervals, the version we use as a baseline in our experiments. The Gittins index \citep{gittins1979bandit} provides an optimal decomposition for discounted rested bandits, and is well characterized for Gaussian iid rewards with an unknown mean estimated from noisy observations \citep{yao2006}. Whittle's index \citep{whittle1988restless} extends index-based policies to restless settings where latent states evolve over time, though it is typically computable only through approximation, and exact optimal policies for restless bandits are PSPACE-hard \citep{papadimitriou1999complexity}. Information-directed sampling \citep{russo2014learning} explicitly targets uncertainty reduction, allocating exploration to maximize information gain per unit of regret; unlike index-based approaches, however, it requires solving an auxiliary optimization problem at each step and does not yield comparative statics in closed form. \citet{russo2018satisficing} propose satisficing Thompson sampling for time-sensitive bandit learning, which samples against a relaxed near-optimal target rather than the exact optimum, and \citet{liu2023nonstationary} extend Thompson sampling to bandits with mean-reverting non-stationary rewards via predictive sampling, which down-weights information that loses value due to drift. Both extend the posterior-sampling framework rather than producing closed-form indices. Predictive sampling is the closer methodological neighbor to our work, but it operates in a different setting (mean-reverting AR(1) rewards) and does not yield the closed-form asymmetry between volatility and stochasticity that we establish for random-walk dynamics. We extend the Gittins analysis to Gaussian state-space bandits with latent random-walk dynamics, and show that the relationship between uncertainty and optimal exploration is not monotonic: volatility promotes exploration while stochasticity suppresses it. We further derive CAUSE, a tractable closed-form index policy that inherits this structure.

\paragraph{Control as inference.}
The connection between optimal control and probabilistic inference has a long history \citep{todorov2008general, kappen2012optimal, levine2018reinforcement}. In this framework, optimal actions are obtained by conditioning on high reward in a probabilistic graphical model, yielding policies that are regularized by their divergence from a prior policy. This formulation has been applied in reinforcement learning \citep{abdolmaleki2018maximum, haarnoja2018soft}, robot planning \citep{toussaint2009robot}, and human planning \citep{botvinick2012planning, piray2021linear, piray2025reconciling}. We leverage this framework to derive a tractable closed-form index for restless bandits with volatility and stochasticity, a setting where the Gittins index characterizes optimal behavior but does not yield a computable policy.

\paragraph{Pathological inference in computational psychiatry.}
Miscalibrated learning-rate adjustment in response to volatility has been proposed as a candidate mechanism in several psychiatric conditions. In anxiety, overestimation of volatility or environmental instability has been linked to inappropriately high learning rates and over-reactivity to noise \citep{browning2015anxious, huang2017computational}, although a recent study has produced inconsistent results \citep{satti2025absence}. In psychosis, miscalibrated precision weighting has been linked to delusional inference \citep{powers2017pavlovian}. In depression, insensitivity to environmental change has been linked to impaired updating of beliefs about reward \citep{pulcu2017affective}. Our framework contributes to this program by providing the exploration-side analogue, with the misallocation pattern characterized in Section~\ref{sec:lesion}.

%% file: 3_problem.tex
\section{Problem setting}
\label{sec:problem}

We consider a restless multi-armed bandit in which each arm's reward is driven by an independent latent state $x_t$ that drifts over time as a Gaussian random walk with innovation variance $v > 0$ (\emph{volatility}). When pulled at time $t$, the arm yields a noisy reward $r_t = x_t + \epsilon_t$ with $\epsilon_t \sim \mathcal{N}(0, s)$ and observation noise variance $s > 0$ (\emph{stochasticity}). The innovation and observation noises are independent across time, and the prior is $x_0 \sim \mathcal{N}(m_0,\, P_0)$. The setting is restless: all arms drift regardless of whether they are pulled. The agent's goal is to maximize the expected discounted return $\mathbb{E}\!\left[\sum_{t=1}^{\infty} \gamma^{t-1} r_t\right]$, where $\gamma \in (0,1)$. Throughout Sections~\ref{sec:optimal} and~\ref{sec:cause}, $v$ and $s$ are treated as known parameters of each arm; Section~\ref{sec:lesion} considers the case in which they are inferred from observations.

To make informed decisions, the agent must track the latent states from noisy observations. The posterior over each arm's state, $p(x_t \mid r_{1:t}) = \mathcal{N}(m_t,\, P_t)$, is updated via the Kalman filter:
\begin{align}
  m_t &= m_{t-1} + K_t\,(r_t - m_{t-1}),
  \label{eq:mean_update} \\
  P_t &= (1-K_t)(P_{t-1} + v),
  \label{eq:var_update}
\end{align}
where $K_t = (P_{t-1} + v) / (P_{t-1} + v + s)$ is the Kalman gain. For arms not pulled, $m_t = m_{t-1}$ and $P_t = P_{t-1} + v$: the mean is unchanged but uncertainty grows by $v$ at every time step.

%% file: 4_motonicity.tex
\section{Volatility, stochasticity, and the value of exploration}
\label{sec:optimal}

We now establish the paper's central theoretical claims: that the exploration bonus depends oppositely on stochasticity and volatility. We first set up the retirement problem and decompose the index, then prove monotonicity in each direction.

\subsection{Exploration bonus via the retirement problem}
\label{sec:index}

To quantify the value of exploring a single arm, we consider a classical thought experiment \citep{gittins1979bandit}. An agent faces a single arm and, at each period, chooses between pulling it (receiving a noisy reward $r_t$ and updating her belief via the Kalman filter) or \emph{retiring} to collect a fixed salary $\lambda$ in perpetuity. The agent may pull as many times as she wishes before retiring, but cannot return once she leaves. The \emph{index} of the arm is the break-even salary: the smallest $\lambda$ at which immediate retirement is optimal. Intuitively, this is the agent's indifference price for giving up future access to the arm, and therefore measures the arm's total value. Formally,
\begin{equation}\label{eq:retire}
  \frac{\lambda}{1 - \gamma}
  \;=\;
  \sup_{\tau \ge 0}\;
  \mathbb{E}\!\left[\;
    \sum_{n=1}^{\tau} \gamma^{n-1}\, r_n
    \;+\;
    \gamma^{\tau}\, \frac{\lambda}{1-\gamma}
  \;\right],
\end{equation}
where the supremum is over all stopping times $\tau$ adapted to the observations. This construction isolates the value of a single arm, independent of the rest of the bandit. In rested bandits, pulling the arm with the highest index is optimal \citep{gittins1979bandit}; in the restless setting, the retirement-problem index remains a well-defined per-arm quantity, and its dependence on $s$ and $v$ is the subject of our theoretical analysis.

The Gaussian structure of the model permits a decomposition of the index into exploitation and exploration components.

\begin{proposition}[Index Decomposition]
\label{prop:decomposition}
The index $\lambda$ admits the decomposition:
\begin{equation}\label{eq:decompose_main}
  \lambda(m,\, P,\, s,\, v,\, \gamma) 
  = m + B(P,\, s,\, v,\, \gamma),
\end{equation}
where $B(P, s, v, \gamma) \coloneqq \lambda(0, P, s, v, \gamma) \ge 0$ is the \textbf{exploration bonus}: the option value of the information gained by pulling.
\end{proposition}

Proven in Appendix~\ref{app:proofs} via shift-coupling of the latent state, the decomposition isolates an exploration term $B$ that depends only on the uncertainty structure. The rest of this section analyzes $B$.

For iid Gaussian bandits with an unknown mean and known observation noise, \citet{yao2006} established that the exploration bonus is nonincreasing in the observation variance and nondecreasing in the prior variance. This setting is the $v=0$ special case of our model. We show below that both monotonicities continue to hold once drift is introduced, and we establish a third monotonicity in $v$ itself, which has no iid counterpart. The opposing effects of $s$ and $v$ on exploration, which is the paper's central structural claim, cannot be drawn within iid models.

\subsection{Exploration should decrease with stochasticity}
\label{sec:mono_s}

\begin{theorem}[Monotonicity in $s$]
\label{thm:mono_s}
The exploration bonus $B(P, s, v, \gamma)$ is nonincreasing in the observation noise $s$.
\end{theorem}

The proof, given in Appendix~\ref{app:proofs}, proceeds by coupling: a noisier arm is constructed from a cleaner one by adding independent Gaussian noise to each observation. Any stopping rule based on noisy observations can be simulated by the clean-arm agent (who can draw the extra noise internally), but not conversely, so the noisy agent optimizes over a strictly smaller class of decision rules. Since the two arms share the same latent state and reward process, both agents are optimizing the same objective; the noisy agent simply has access to fewer decision rules, so the supremum is no larger, and the monotonicity follows. Because stochasticity enters only through the observation channel, which is active only when the arm is pulled, this argument does not rely on whether unselected arms drift; the result holds in both rested and restless settings.





\subsection{Exploration should increase with volatility}
\label{sec:mono_v}

\begin{theorem}[Monotonicity in $v$]
\label{thm:mono_v}
The exploration bonus $B(P, s, v, \gamma)$ is nondecreasing in the innovation variance $v$.
\end{theorem}

Where stochasticity degrades the signal, volatility creates opportunity. The proof, given in Appendix~\ref{app:proofs}, rests on a supporting lemma that is of independent interest:

\begin{lemma}[Monotonicity in $P$]
\label{lem:mono_P}
The exploration bonus $B(P, s, v, \gamma)$ is nondecreasing in the posterior variance $P$.
\end{lemma}

The mechanism is convexity. Higher posterior variance $P$ produces a wider distribution over the next-period posterior mean, because more of the remaining uncertainty will be resolved by the next observation. Since the agent can exploit favorable posterior means and retire to avoid unfavorable ones, the value function is convex in the posterior mean, and Jensen's inequality converts wider spread into higher expected value. More uncertainty about the arm therefore yields a larger option value of pulling.

Volatility enters this logic through the belief dynamics. Each unit of $v$ adds directly to the next-period posterior variance ($P'$ grows with $v$) and, by Lemma~\ref{lem:mono_P}, raises the bonus. The effect compounds over time: in a volatile environment, each future pull operates on an arm whose uncertainty has been replenished by drift, so the option value of pulling is sustained rather than decaying toward zero.

%% file: 5_index.tex
\section{CAUSE: a closed-form index via control-as-inference}
\label{sec:cause}

The analysis of Section~\ref{sec:optimal} characterizes how exploration should vary when faced with both volatility and stochasticity, but does not yield a computable policy for the restless setting. We now derive CAUSE, a tractable closed-form index for Gaussian state-space bandits, by casting action selection as posterior inference under an optimality constraint within the control-as-inference framework \citep{levine2018reinforcement, todorov2008general}. The derivation is methodologically distinct from standard routes to exploration bonuses (optimism, posterior sampling, dynamic programming on belief-state MDPs). It proceeds in three steps: a binary optimality variable is attached to each future reward; conditioning on all of them yields a backward message over the current latent state, which we approximate within the sigmoid family; and multiplying this message into the predictive belief gives an index that separates into the current estimate and an exploration bonus. The complete derivation is provided in Appendix~\ref{app:derivation}.

\subsection{Control as inference for the state-space bandit}

Following the control-as-inference framework, we place a binary optimality variable $o_t \in \{0, 1\}$ at each time step with
\begin{equation}\label{eq:opt_likelihood}
  p(o_t = 1 \mid r_t) = \sigma(\gamma^{t-1}\, r_t),
\end{equation}
where $\sigma$ is the logistic sigmoid and the discount $\gamma^{t-1}$ encodes the diminishing importance of future rewards. The agent's action preference is governed by the posterior $p(x_t \mid o_{t:\infty} = 1, r_{1:t-1})$, which weighs the predictive belief about $x_t$ against the backward message $b_t(x_t) \equiv p(o_{t:\infty} = 1 \mid x_t)$ carrying information from future optimality constraints. Exploration arises not from an externally added bonus but from this posterior: states that make future optimality more likely are upweighted, and the corresponding arms preferred.

A note on what the inference reformulation buys: in the optimality formulation, restless bandits are notoriously difficult because optimal actions depend on the joint state of all arms \citep{whittle1988restless}. Under the inference formulation, this difficulty does not arise: each arm's posterior under future optimality depends only on its own latent process and backward message, and action selection reduces to comparing per-arm scalars. The restless structure, hard in the optimality formulation, becomes per-arm in the inference formulation, which is what makes a closed-form index possible.

\subsection{The closed-form index}

The backward message admits a closed-form recursion in the Gaussian state-space setting. We restrict the backward message to a sigmoid ansatz $b_t(x_t) \approx \sigma(\alpha_t x_t)$, natural because the local optimality likelihood is itself sigmoid (Eq.~\ref{eq:opt_likelihood}), and track the coefficient $\alpha_t$ through the recursion using probit approximations for the Gaussian-sigmoid convolutions involved (Appendix~\ref{app:derivation}). This yields an infinite-horizon precision $S$ that depends transparently on $\gamma$, $s$, and $v$:
\begin{equation}\label{eq:S_main}
  S \approx \frac{2}{(\ln D+1-\gamma)\bigl(1 + \sqrt{1 + \phi\, s}\bigr)},
  \qquad 
  D = \tfrac{1}{2}\bigl(1 + \sqrt{1 + \phi\, v\, (\alpha^*)^2}\bigr),
\end{equation}
with $\phi = \pi/8$ and $\alpha^* = 2/[(1-\gamma)(1 + \sqrt{1 + \phi s})]$ the undamped ($v=0$) limit. The structure is interpretable: stochasticity $s$ damps $S$ through the $\sqrt{1 + \phi s}$ factor, while volatility $v$ enters through $D$, which shortens the effective planning horizon by attenuating the backward flow of precision. Combining $S$ with the predictive belief and approximating the resulting skew-normal posterior mean (Appendix~\ref{app:derivation}) gives the CAUSE index:
\begin{equation}\label{eq:index_closed}
  \boxed{\;\lambda_{\mathrm{CAUSE}} = m + c\, (P+v)\, \tilde{\alpha},\;}
\end{equation}
where
\begin{equation}\label{eq:alpha_tilde}
  \tilde{\alpha} = \frac{S}{\sqrt{1 + \phi\, (P+v)\, S^2}},
\end{equation}
$m$ and $P$ are the current Kalman posterior moments, and $c \in (0, 1)$ is a tunable scale parameter. The agent selects the arm with the largest $\lambda_{\mathrm{CAUSE}}$. The parameter $c$ controls the overall magnitude of the exploration bonus and does not affect its dependence on $s$ or $v$; we fix $c = 1/2$ throughout this paper, a value that emerges naturally from one route through the derivation (Appendix~\ref{app:derivation}).

\subsection{Structure of the index}

Stochasticity $s$ damps $S$ through the $(1 + \sqrt{1 + \phi s})^{-1}$ factor, so $S$, $\tilde\alpha$, and the bonus all decrease with $s$. In the high-volatility (or high-$P$) limit $(P+v) \to \infty$, $\tilde{\alpha} \to 1/\sqrt{\phi\, (P+v)}$ and the bonus scales as $c\sqrt{8/\pi}\sqrt{(P+v)}$, recovering canonical UCB-style $\sqrt{P+v}$ scaling. More generally, volatility raises the bonus through the predictive variance $P+v$ while damping $S$ through $D$, acting as a dual discount that shortens the effective planning horizon. The CAUSE index thus inherits, in closed form, the structural properties of the optimal Gittins bonus: exploration rises with volatility that makes new observations relevant and falls with stochasticity that makes them unreliable.

%% file: 6_results.tex
\section{Experiments}
\label{sec:experiments}

We evaluate CAUSE in four experiments: cumulative regret in heterogeneous-arms restless bandits, comparison of CAUSE’s exploration bonus to the numerically computed Gittins bonus along the stochasticity axis, comparison against exactly computed optimal policies at $v = 0$ and at $K = 2$, and a lesion analysis that uses CAUSE to characterize the behavioral consequences of miscalibrated noise inference.

\subsection{Ignoring the cause of uncertainty is costly}
\label{sec:regret}
Standard exploration strategies treat uncertainty as a single quantity and prescribe more exploration whenever uncertainty is high. The structural analysis of Section~\ref{sec:optimal} predicts this is suboptimal whenever volatility and stochasticity vary across arms: uncertainty arising from volatility warrants exploration; uncertainty arising from stochasticity does not. We test this in three regimes. In the \emph{mixed} regime, arms span the four cells of the $(v, s)$ grid, with both volatility and stochasticity varying across arms. In the \emph{s-dominant} regime, stochasticity varies substantially across arms while volatility is uniform. In the \emph{v-dominant} regime, volatility varies while stochasticity is uniform. We compare CAUSE against Thompson sampling \citep{thompson1933likelihood}, UCB \citep{kaufmann2012bayesian}, predictive sampling \citep{liu2023nonstationary}, IDS (with the information ratio minimized numerically) \citep{russo2014learning}, and two index policies obtained by value iteration: a Gittins-per-arm policy, whose bonus solves the rested retirement problem, and the Whittle index, which solves the single-arm subsidy problem \citep{whittle1988restless}. The last three admit no closed form, since each requires numerically solving an auxiliary problem, whereas the CAUSE index is evaluated directly from Equation~\ref{eq:index_closed}. We also include a myopic baseline, which performs no exploration (Appendix~\ref{app:setup_baselines}).

\begin{figure}[H]
  \centering
  \includegraphics[width=\linewidth]{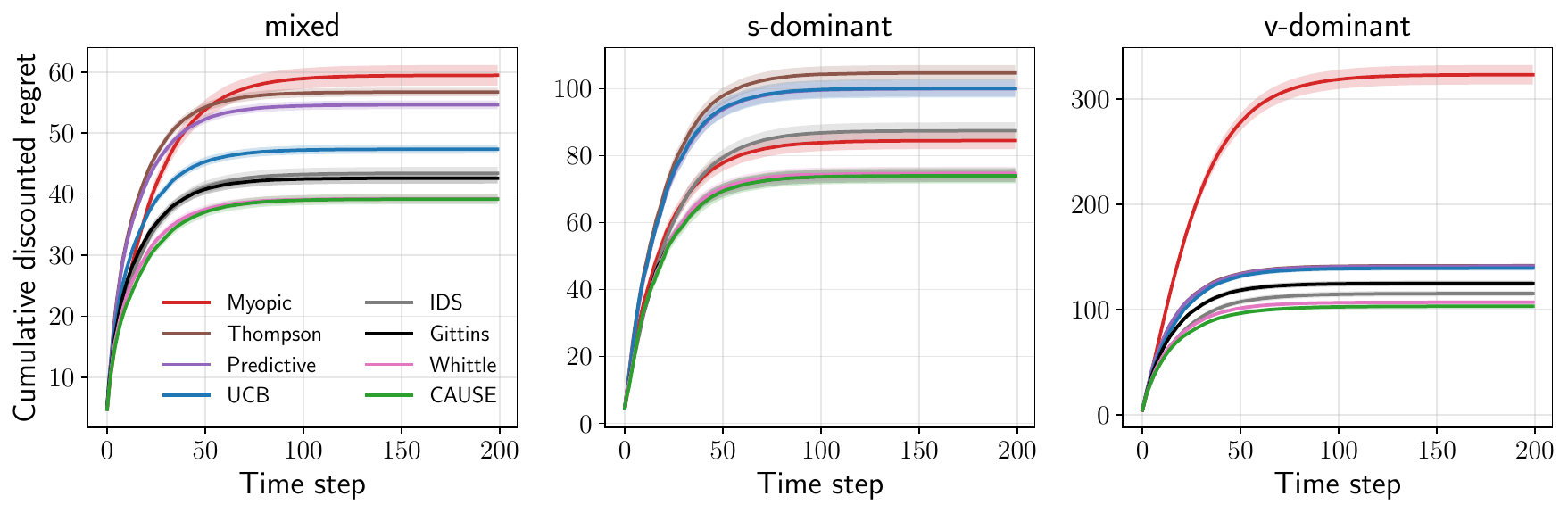}
  \caption{Cumulative discounted regret over $T = 200$ steps in three regimes ($K = 4$, $\gamma = 0.95$, 1000 Monte Carlo runs). CAUSE is best or statistically tied with the best in all three regimes.}
  \label{fig:regret}
\end{figure}

In the mixed regime (Figure~\ref{fig:regret}, left), CAUSE achieves the lowest regret, statistically tied with the Whittle index and ahead of Gittins-per-arm. This is not surprising: the Gittins index is optimal for rested bandits, but its rested-optimality does not transfer to the restless setting we study, in which the latent state evolves under volatility. The Whittle index, by contrast, is constructed for restless problems and does not incur this penalty. In the $s$-dominant regime (Figure~\ref{fig:regret}, middle), where stochasticity varies substantially across arms but volatility is shared, we expect all three index policies to perform well. High-$s$ arms accumulate high posterior variance, but exploring such arms is unproductive: the variance arises from stochasticity rather than from drift in the latent state, so additional observations reduce posterior uncertainty only slowly. All three correctly down-weight exploration of high-$s$ arms, and the figure confirms this: the three policies tie at the lowest regret. UCB and Thompson sampling, which treat posterior variance as a single quantity, incur substantially higher regret, worse than the myopic baseline that ignores uncertainty entirely, as does IDS. In the $v$-dominant regime (Figure~\ref{fig:regret}, right), CAUSE outperforms both Gittins-per-arm and the Whittle index. CAUSE and Gittins-per-arm both allocate more exploration to high-volatility arms, as the structural analysis prescribes; the difference is that Gittins-per-arm computes its bonus under rested dynamics, in which unselected arms do not drift, and that mismatch is largest precisely when volatility is high. CAUSE's bonus additionally damps the effective horizon with volatility (Section~\ref{sec:cause}), though see the ablation in Appendix~\ref{app:numS}. The Whittle index, which does account for the restless dynamics, closes much of this gap but remains behind CAUSE.

\begin{table}[H]
  \centering
  \caption{Final cumulative discounted regret across problem sizes
  ($T = 200$, $\gamma = 0.95$, mean across 1000 Monte Carlo runs).
  Bold marks the best policy and any policy statistically tied with it.}
  \label{tab:regret_K}
  \resizebox{\linewidth}{!}{%
  \begin{filecontents*}{figures/regret_K_table_tex.csv}
K,regime,Myopic,Thompson,Predictive,UCB,IDS,Gittins,Whittle,CAUSE
4,mixed,59.47,56.72,54.63,47.36,43.42,42.61,\textbf{39.25},\textbf{39.18}
,s-dominant,84.53,104.76,100.03,100.13,87.46,\textbf{74.06},\textbf{74.73},\textbf{74.04}
,v-dominant,323.29,141.78,141.37,139.48,115.37,125.02,107.07,\textbf{103.30}
8,mixed,113.02,99.80,95.68,78.91,74.95,\textbf{71.99},\textbf{69.23},\textbf{70.13}
,s-dominant,127.06,169.43,154.05,142.95,113.50,\textbf{108.64},\textbf{110.59},\textbf{106.12}
,v-dominant,600.66,256.98,259.51,217.84,205.69,\textbf{182.94},\textbf{178.61},\textbf{182.42}
12,mixed,151.23,131.34,128.02,102.79,100.08,\textbf{94.40},\textbf{92.15},\textbf{94.15}
,s-dominant,156.66,206.78,190.15,174.25,\textbf{129.41},\textbf{131.79},135.61,\textbf{130.91}
,v-dominant,803.44,351.33,351.22,304.44,291.95,\textbf{256.20},\textbf{253.84},\textbf{259.59}
16,mixed,180.25,157.55,152.96,119.06,116.66,\textbf{107.46},\textbf{111.28},\textbf{110.26}
,s-dominant,180.02,239.60,216.24,192.62,\textbf{144.06},\textbf{144.84},150.74,\textbf{144.04}
,v-dominant,962.98,432.16,433.90,374.12,357.42,\textbf{317.76},\textbf{314.20},\textbf{321.54}
\end{filecontents*}
\csvreader[
    tabular = ll*{8}{c},
    table head = \toprule $K$ & Regime & Myopic & Thompson & Predictive & UCB & IDS & Gittins & Whittle & CAUSE \\ \midrule,
    table foot = \bottomrule,
    late after line = \\\noalign{\ifcsvstrcmp{\regime}{mixed}{\vskip2pt\hrule height0.4pt\vskip2pt}{}},
    late after last line = \\,
    head to column names
  ]{figures/regret_K_table_tex.csv}{}{\K & \regime & \Myopic & \Thompson & \Predictive & \UCB & \IDS & \Gittins & \Whittle & \CAUSE}}
\end{table}

At larger problem sizes ($K$), CAUSE remains either the best policy or statistically tied with the best in every regime, as assessed by max-t paired bootstrap (Table~\ref{tab:regret_K}; Appendix~\ref{app:k_table}). Robustness across discount factors ($\gamma$), the scale parameter ($c$), and the UCB scaling parameter is reported in Appendix~\ref{app:robustness}.

\subsection{CAUSE tracks the shape of the Gittins bonus}
\label{sec:bonus_gittins}

\begin{wrapfigure}[14]{r}{0.4\linewidth}
\vspace{-1em}
  \centering
  \captionsetup{skip=2pt}
  \includegraphics[width=\linewidth]{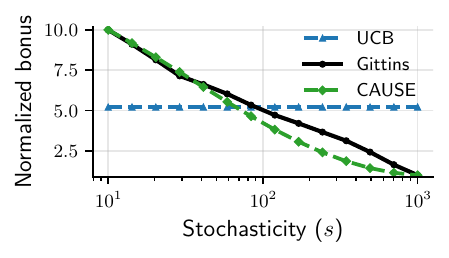}
  \caption{Exploration bonus as a function of stochasticity $s$, normalized to a common range. CAUSE tracks the Gittins shape; UCB is insensitive to $s$.}
  \label{fig:bonus}
\end{wrapfigure}
CAUSE is derived independently of the Gittins framework, via control-as-inference under an optimality constraint. Whether its closed form captures the structural shape of the optimal Gittins bonus is an empirical question. We compute the Gittins bonus by value iteration on the retirement problem (Appendix~\ref{app:setup_gittins}) and compare it to the bonuses of CAUSE and UCB at a fixed posterior variance (Thompson sampling and predictive sampling do not produce explicit per-arm bonuses).

Figure~\ref{fig:bonus} shows each policy's bonus along the stochasticity axis. CAUSE tracks the Gittins shape closely across two orders of magnitude, while UCB is insensitive to $s$ by construction, allocating exploration based on posterior variance alone. The convergence of two structurally distinct derivations (the Gittins retirement problem and the control-as-inference posterior) on the same $s$-dependence supports the structural account of Section~\ref{sec:optimal}. Bonus dependence on volatility (Appendix~\ref{app:bonus_v}) is less discriminative between policies; the differences along that axis emerge in regret rather than in bonus shape.

\subsection{CAUSE matches the optimum wherever it is computable}
\label{sec:rested_verification}

CAUSE's derivation involves several analytical approximations (Appendix~\ref{app:derivation}). Quantifying the suboptimality these approximations introduce is intractable in general, since formal regret bounds for index policies on restless Gaussian bandits remain an open problem. We therefore compare CAUSE against exactly computed optimal policies in the two settings where one is available, chosen so that between them they exercise every approximation in the derivation.

At the rested limit ($v = 0$), the optimal policy is the numerically computable Gittins index, provably optimal for infinite-horizon discounted reward in iid bandits \citep{gittins1979bandit}. This regime admits a direct comparison of CAUSE against optimal performance, and it is not a degenerate test of the derivation: the two least principled steps, the sigmoid ansatz for the backward message and the replacement of the gate by a constant, are both active at $v = 0$, as are the probit approximations, whose error is standard and bounded (Appendix~\ref{app:derivation}). At $v = 0$ across both moderate ($s \in \{9, 25\}$) and extreme ($s \in \{9, 900\}$) heterogeneity in stochasticity, CAUSE's discounted cumulative regret is statistically indistinguishable from the numerical Gittins reference. At $s \in \{9, 25\}$ over $T = 200$ steps and $1000$ Monte Carlo runs, CAUSE achieves $22.61 \pm 0.88$ and Gittins-per-arm achieves $22.50 \pm 0.75$ (mean $\pm$ SEM), with the difference of $0.11$ well within the sampling error. The closeness is preserved at extreme stochasticity ($s \in \{9, 900\}$, Appendix~\ref{app:v0_verification}).

The remaining three steps collapse the backward recursion into closed form and are active only when the latent state drifts, so they require a comparison at $v > 0$. At $K = 2$, exact value iteration over the joint belief state is feasible and yields the numerically optimal restless policy (Appendix~\ref{app:optk2}). Across six two-arm configurations spanning the three regimes, CAUSE is statistically indistinguishable from this optimum in every case; in the two mixed configurations, for instance, CAUSE achieves $17.05 \pm 0.64$ and $17.03 \pm 0.67$ against $16.79 \pm 0.62$ and $16.61 \pm 0.63$ for the optimal policy. UCB, Thompson sampling, and predictive sampling, by contrast, trail the optimum by 10\% to 47\% across the same configurations. 

The approximations exercised by each of these comparisons are summarized in Table~\ref{tab:approx_coverage}; Appendix~\ref{app:numS} additionally reports a variant in which $S$ is obtained by numerically iterating Eq.~\ref{eq:alpha-recursion}, isolating the four steps that the closed form replaces. Combined with the restless results of Section~\ref{sec:regret}, this provides a layered empirical bound on CAUSE's suboptimality: where the optimal policy is known, at $v = 0$ and at $K = 2$, CAUSE matches it within Monte Carlo precision; where the optimal policy is intractable, CAUSE is best or statistically tied with the best across the regimes and problem sizes tested.

\subsection{Pathological noise inference produces opposing shifts in exploration}
\label{sec:lesion}

We compare three agents that differ only in their inference of the environmental noise structure. The \emph{healthy} agent jointly infers $(\hat s, \hat v)$ from the observation stream by explaining away \citep{piray2021model}: the two sources compete to account for the observed variance, and their relative contributions are therefore disentangled by their distinct temporal signatures. The \emph{stochasticity-blind} agent, by contrast, is insensitive to $s$; its $\hat s$ estimate is pinned at its initial value, so the observed variance is attributed almost entirely to volatility, producing a small $\hat s$ and an inflated $\hat v$. The \emph{volatility-blind} agent is the symmetric case: $\hat v$ is pinned at its initial value, so the variance is instead absorbed into $\hat s$. Each agent then acts on its own inferred $(\hat s, \hat v)$ and its own evolving posterior. We evaluate the three agents in the $s$-dominant and $v$-dominant regimes of Section~\ref{sec:regret}, which vary stochasticity and volatility respectively while holding the other fixed, so that each regime isolates one axis of inference.

We first consider learning rate (or Kalman gain) as defined in Equation~1. We have previously shown that learning rate is not merely disrupted in each blind model, but rather reverses the sign of the healthy modulation on its own axis \citep{piray2021model}. This follows because each blind agent computes its gain from its own inferred $(\hat s, \hat v)$ and therefore substitutes one source for the other: as the true value on an axis grows, the estimate on the orthogonal axis grows in its place, since the variance that cannot be assigned to the pinned source must be absorbed by the free one. The gain is moreover increasing in $\hat v$ and decreasing in $\hat s$, so the substitution inverts the direction of modulation rather than merely attenuating it. Now we consider the rate of exploratory choices, defined as the fraction of trials on which the agent selects an arm other than the one with the highest posterior mean. This is the action-selection signal, and it occurs on 7--14\% of trials. 

Unlike the learning rate, the rate of exploratory choices does not reverse. Stochasticity enters the gain directly, whereas the posterior variance it drives grows only as $\sqrt{s}$, so for a bounded ratio such as the gain the inference channel dominates. The bonus instead grows with the posterior variance, and at equilibrium uncertainty that channel dominates in turn. Each lesion therefore shifts allocation rather than reversing it. Nevertheless, both blind models show selective maladjustment of exploration on the corresponding axis of inference. The stochasticity-blind agent explores high-$s$ arms $2.6\times$ as often as the healthy agent ($+0.0349$, CI $[+0.0278, +0.0421]$), and therefore explores more in the regime where observations are least informative. The volatility-blind agent, conversely, explores high-$v$ arms 37\% less often ($-0.0064$, CI $[-0.0081, -0.0047]$) and reallocates toward low-$v$ arms ($+0.0068$), and therefore explores less in the regime where observations are most informative. 

Appendix~\ref{app:lesion_fixedP} reports the same agents at a fixed reference posterior variance, where each lesion reverses its corresponding monotonicity, as it does for the learning rate. Holding the posterior variance fixed isolates the inference channel from the growth in posterior variance that stochasticity itself drives, which dominates realized allocation in closed loop.

%% file: 7_discussion.tex
\section{Discussion}
\label{sec:discussion}
We have shown that volatility and stochasticity drive the value of exploration in opposite directions, that this asymmetry is captured in closed form by the CAUSE index, and that miscalibrated noise inference produces dissociable, axis-specific misallocation of exploration whose direction identifies the failed inferential channel. The structural result extends the Gittins framework from iid Gaussian bandits to Gaussian state-space bandits with latent dynamics. Empirically, CAUSE matches or exceeds every baseline and reference we test, including the optimal policy wherever it can be computed.

The framework is developed for Gaussian state-space models with random-walk latent dynamics, and the closed-form CAUSE index inherits this scope; extension to non-Gaussian distributions or change-point dynamics is an open question. The directional predictions should extend on general informational grounds (uncertainty arising from drift is reducible by exploration, uncertainty arising from stochasticity is not, regardless of the parametric form), but the closed form itself relies on Gaussian conjugacy and the probit approximation to Gaussian-sigmoid convolutions, and does not transfer mechanically. The control-as-inference framework we use to derive CAUSE, however, is general: it casts exploration as a backward message-passing problem that applies wherever a Bellman recursion admits this form. To our knowledge, this is the first application of control-as-inference to derive an exploration index for restless bandits, and the framework provides a route to closed-form policies in restless bandit classes beyond linear-Gaussian dynamics. Within the present setting, our derivation tracks the backward message in the sigmoid family using probit approximations and trapezoidal linearization of the recursion (Appendix~\ref{app:derivation}); an exponential-family alternative admits exact Gaussian recursions but is blind to the distinction between the two noise sources, and as a full agent performs worse than a policy that never explores (Appendix~\ref{app:exp}). The sigmoid choice trades exact tractability for calibrated behavior. The closed form rests on several approximations whose error we do not bound analytically. Where optimal exploration can be computed by dynamic programming, in the rested problem ($v = 0$) and at $K = 2$ for $v > 0$, CAUSE is statistically tied with the optimal policy. Despite requiring no value iteration, CAUSE matches or exceeds the numerically computed index policies at $K = 4$ and is statistically tied with them at larger problem sizes. 

The predicted misallocation of exploration under pathological noise inference offers concrete empirical targets for behavioral testing, complementing the learning-rate signatures previously characterized \citep{fang2026inferring, fang2026dissociating, chen2026categorical}. These predictions are related most directly to computational accounts of psychiatric conditions in which inference about environmental noise has been implicated, where the empirical exploration literature remains mixed \citep{aylward2019altered, fan2023trait}.

%% file: Appendix_A.tex
\newpage
\section{Proofs of monotonicity results}
\label{app:proofs}

This appendix provides proofs of the three monotonicity results stated in Section~\ref{sec:optimal} of the main text: the index decomposition (Proposition~\ref{prop:decomposition}), the monotonicity of the exploration bonus in the observation noise $s$ (Theorem~\ref{thm:mono_s}), and the monotonicity in the innovation variance $v$ (Theorem~\ref{thm:mono_v}). The monotonicity in the posterior variance $P$ (Lemma~\ref{lem:mono_P}) follows from \citet{yao2006} and is proved here for use in the $v$-proof. Notation, the generative model, and the Kalman recursion follow Section~\ref{sec:problem} of the main text.

\subsection{Preliminaries}
\label{app:prelim}

A \emph{stopping rule} for an observation sequence $r_1, r_2, \ldots$ is a rule that, at each time $n$, decides whether to stop using only the observations $r_1, \ldots, r_n$ seen so far. We write $\xi$ for the time at which the rule stops.

\begin{definition}[Value function of the retirement problem]
\label{def:value}
Given retirement salary $\lambda$, the value function of the retirement problem is
\[
  V(m, P; s, v, \gamma)
  = \sup_{\xi \ge 0}\; E\!\left[\left.
    \sum_{n=1}^{\xi} \gamma^{n-1}\, r_n
    + \gamma^{\xi}\, \frac{\lambda}{1-\gamma}
  \;\right|\; x_0 \sim N(m, P)\right],
\]
where the supremum is over stopping rules $\xi$ for the observation sequence.
\end{definition}

The value function satisfies the Bellman equation
\begin{equation}\label{eq:bellman}
  V(m, P) = \max\!\left(\frac{\lambda}{1-\gamma},\;\; m + \gamma\, E_{m'}[V(m', P')]\right),
\end{equation}
where $P' = (P+v)s/(P+v+s)$ is the updated posterior variance (deterministic) and $m' \sim N(m, \sigma^2(P))$ with $\sigma^2(P) = (P+v)^2/(P+v+s) = (P+v) - P'$ is the next-step posterior mean.

\subsection{Index decomposition}
\label{app:index_decomp}

Rearranging the retirement equation~\eqref{eq:retire} using $\sum_{n=1}^{\xi} \gamma^{n-1} = (1-\gamma^\xi)/(1-\gamma)$ yields the equivalent ratio form
\begin{equation}\label{eq:ratio}
  \lambda
  \;=\;
  \sup_{\xi \ge 1}\;
  \frac{E\!\left[\sum_{n=1}^{\xi} \gamma^{n-1}\, r_n \right]}
       {E\!\left[\sum_{n=1}^{\xi} \gamma^{n-1} \right]},
\end{equation}
which we use in the proof of Theorem~\ref{thm:mono_s} below.

\begin{proof}[Proof of Proposition~\ref{prop:decomposition}]
Replacing the prior mean $m$ by $m + a$ and writing $\tilde{x}_t = x_t - a$ gives $\tilde{x}_0 \sim N(m, P)$, with the same innovations and observation noise, and $r_t = \tilde{x}_t + a + \epsilon_t$. The filtration is unchanged. Substituting $r_n = \tilde{r}_n + a$ into~\eqref{eq:retire}:
\[
  E\!\left[\sum_{n=1}^{\xi} \gamma^{n-1} r_n + \gamma^\xi \frac{\lambda}{1-\gamma}\right]
  = E\!\left[\sum_{n=1}^{\xi} \gamma^{n-1} \tilde{r}_n + \gamma^\xi \frac{\lambda - a}{1-\gamma}\right]
     + \frac{a}{1-\gamma}.
\]
The retirement equation at salary $\lambda$ for the shifted arm is thus the retirement equation at salary $\lambda - a$ for the unshifted arm, after subtracting $a/(1-\gamma)$ from both sides. Hence $\lambda(m+a, P, s, v, \gamma) = a + \lambda(m, P, s, v, \gamma)$, and setting $m = 0$ gives~\eqref{eq:decompose_main}.
\end{proof}

\subsection{Proof of Theorem~\ref{thm:mono_s}}
\label{app:mono_s}

\begin{proof}
The argument adapts Yao's randomization proof for iid Gaussian bandits \citep[Lemma~1]{yao2006} to the state-space setting. Fix $s_1 < s_2$, and let $\{x_t\}$ be a single realization of the latent state. Let $\epsilon_n \sim \mathcal{N}(0, s_1)$ be the observation noise of the clean arm, $r_n = x_n + \epsilon_n$, and $\eta_n \sim \mathcal{N}(0, s_2 - s_1)$ independent of everything else. Setting $r_n' = r_n + \eta_n$ gives a valid observation sequence for an $s_2$-arm driven by the same latent process.

Any stopping rule $\xi'$ on the noisier sequence $\{r_n'\}$ can be viewed as a randomized stopping rule on the clean sequence $\{r_n\}$: the agent operating on $\{r_n\}$ can simply draw the extra noise $\{\eta_n\}$ internally and stop according to $\xi'$. The supremum in the ratio form (Eq.~\eqref{eq:ratio}) over stopping rules for $\{r_n'\}$ is therefore bounded above by the supremum over (possibly randomized) stopping rules for $\{r_n\}$, giving $B(P, s_2, v, \gamma) \le B(P, s_1, v, \gamma)$.

The randomization argument extends from the iid setting to the state-space setting without modification because the observation channel is still additive Gaussian noise and the latent process $\{x_t\}$ is shared between the clean and noisy versions of the bandit; the latent dynamics never enter the coupling.
\end{proof}

\begin{remark}
The argument uses no property of the latent dynamics $\{x_t\}$; it applies to both rested and restless bandits. The key requirements are additive observation noise and independence of $\eta_n$ across time.
\end{remark}

\subsection{Proof of Lemma~\ref{lem:mono_P}}
\label{app:mono_P}

\begin{proof}
The result follows from a standard convexity argument. The retirement-problem value function $V(m, P)$ is convex in $m$ for each $P$ (a standard backward-induction property of Bellman operators with linear and constant alternatives, used also in the proof of Theorem~\ref{thm:mono_v}). The $P$-dependent Bellman operator at fixed $v$ satisfies: increasing $P$ increases the predictive variance $P + v$, hence increases the spread of the next-step posterior mean $m'$, hence by Jensen's inequality increases $\mathbb{E}[V(m', P')]$ at fixed continuation value. The break-even $\lambda$ inherits the monotonicity. In the iid limit ($v = 0$), this reduces to a result of \citet{yao2006} (Theorem~1).
\end{proof}

\subsection{Proof of Theorem~\ref{thm:mono_v}}
\label{app:mono_v}

\paragraph{Intuition.}
When the latent state drifts faster, two things happen. First, each pull is more informative relative to the growing uncertainty, so exploring is more rewarding. Second, the arm keeps generating fresh uncertainty even after you observe it, effectively keeping the posterior variance elevated, and higher variance increases the bonus by Lemma~\ref{lem:mono_P}.

\begin{proof}
We track the $v$-dependence of the value function explicitly, writing $V(m, P; v)$. Recall
\[
  P'(v) = \frac{(P+v)s}{P+v+s}, \qquad \sigma^2(v) = \frac{(P+v)^2}{P+v+s},
\]
both strictly increasing in $v$.

Convexity of $V$ in its first argument follows from a standard backward-induction argument: the continuation value $m + \gamma\, E[V(m', P'; v)]$ with $m' \sim N(m, \sigma^2)$ is a Gaussian convolution of a convex function plus a linear term, hence convex, and the Bellman operator preserves convexity through the pointwise maximum with the retirement value.

The proof structure is: $v$ enters the continuation value at three places, through the spread $\sigma^2(v)$ of the next-step posterior mean, through the next-step posterior variance $P'(v)$, and through the next-period value function $V(\cdot, \cdot; v)$ itself. Each contributes a monotonicity-preserving inequality, which we chain.

\smallskip\noindent\emph{Base case.} $V = \lambda/(1-\gamma)$, independent of $v$.

\smallskip\noindent\emph{Inductive step.} Joint hypothesis: at all future stages, $V(\cdot, P; v)$ is convex in its first argument for every $(P, v)$, and $V(m, P; v)$ is nondecreasing in $v$ for every $P$. Take $v_1 < v_2$ and consider the continuation value $C(0, P; v) = \gamma\, E[V(m'_v, P'(v); v)]$ where $m'_v \sim N(0, \sigma^2(v))$. We chain three inequalities:

\begin{enumerate}
\item \emph{Wider spread of $m'$ helps (Jensen).} $\sigma^2(v_2) > \sigma^2(v_1)$ and $V(\cdot, P'(v_2); v_2)$ is convex, so spreading the zero-mean distribution of $m'$ increases its expected value:
\[
  E[V(m'_{v_2}, P'(v_2); v_2)] \ge E[V(m'_{v_1}, P'(v_2); v_2)].
\]

\item \emph{Larger $P'$ helps (Lemma~\ref{lem:mono_P}).} $P'(v_2) > P'(v_1)$ and $V$ is nondecreasing in its second argument:
\[
  E[V(m'_{v_1}, P'(v_2); v_2)] \ge E[V(m'_{v_1}, P'(v_1); v_2)].
\]

\item \emph{Larger $v$ helps (inductive hypothesis).} $v_2 > v_1$ and $V$ is nondecreasing in $v$:
\[
  E[V(m'_{v_1}, P'(v_1); v_2)] \ge E[V(m'_{v_1}, P'(v_1); v_1)].
\]
\end{enumerate}

Chaining: $C(0, P; v_2) \ge C(0, P; v_1)$, and therefore
\[
  V(0, P; v_2) = \max\!\left(\frac{\lambda}{1-\gamma},\; C(0, P; v_2)\right) \ge V(0, P; v_1).
\]
By Proposition~\ref{prop:decomposition}, the bonus inherits this monotonicity at all $m$, hence $B(P, s, v_2, \gamma) \ge B(P, s, v_1, \gamma)$.
\end{proof}

%% file: Appendix_B.tex
\newpage
\section{Derivation of the CAUSE index}
\label{app:derivation}

We derive the closed-form exploration index used in the main text. The derivation proceeds in four steps: (i) marginalize the reward out of the per-step optimality likelihood, (ii) solve the backward recursion under a sigmoid ansatz to obtain a closed-form infinite-horizon message, (iii) form the posterior over the latent state under future optimality, and (iv) extract the index as the posterior mean, replacing a single gating term by a tunable scale.

\subsection{Generative model}

Each arm is governed by a linear-Gaussian state-space model with latent dynamics and observation model
\begin{equation}
  p(x_{t+1} \mid x_t) = \mathcal{N}(x_{t+1} \mid x_t, v),
  \qquad
  p(r_t \mid x_t) = \mathcal{N}(r_t \mid x_t, s).
\end{equation}
Following the control-as-inference framework, we place a binary optimality variable at each step,
\begin{equation}\label{eq:opt_likelihood}
  p(o_t = 1 \mid r_t) = \sigma(\gamma^{t-1} r_t),
\end{equation}
where $\sigma$ is the logistic sigmoid and $\gamma \in (0,1)$ is the discount factor. The agent acts under the event $o_{t:\infty} = 1$.

\subsection{Per-step optimality after marginalizing the reward}

Marginalizing the unobserved reward against the observation model and using the standard probit approximation for the Gaussian-sigmoid convolution \citep{mackay2003information} yields
\begin{equation}\label{eq:local}
  p(o_t = 1 \mid x_t) \;=\; \int \sigma(\gamma^{t-1} r_t)\, \mathcal{N}(r_t \mid x_t, s)\, dr_t \;\approx\; \sigma(\beta_t\, x_t),
\end{equation}
with
\begin{equation}
  \beta_t = \frac{\gamma^{t-1}}{\sqrt{1 + \phi\, \gamma^{2(t-1)}\, s}},
  \qquad \phi = \frac{\pi}{8}.
\end{equation}
Observation noise $s$ dampens the local optimality signal: larger $s$ weakens the coupling between latent state and optimality at that step.

\subsection{Backward message recursion}

The backward message is the probability of future optimality conditioned on the current latent state,
\begin{equation}
  b_t(x_t) \equiv p(o_{t:\infty} = 1 \mid x_t),
\end{equation}
which satisfies
\begin{equation}\label{eq:recursion}
  b_t(x_t) = p(o_t = 1 \mid x_t) \int p(x_{t+1} \mid x_t)\, b_{t+1}(x_{t+1})\, dx_{t+1}.
\end{equation}
Exact evaluation of this recursion is intractable. We adopt the ansatz
\begin{equation}\label{eq:ansatz}
  b_t(x_t) \;\approx\; \sigma(\alpha_t\, x_t),
\end{equation}
where $\alpha_t$ is a scalar coefficient to be determined, and track its value through time. This is in the spirit of fixed-form variational inference: we restrict the backward message to a parametric family and track the parameter that best represents the true message within that family.

\paragraph{Base case.}
At the terminal horizon $T$, the future message is trivial and Eq.~\eqref{eq:recursion} reduces to the local likelihood, giving $b_T(x_T) = \sigma(\beta_T x_T)$ and thus $\alpha_T = \beta_T$.

\paragraph{Inductive step.}
Assume $b_{t+1}(x_{t+1}) \approx \sigma(\alpha_{t+1} x_{t+1})$. Applying the probit approximation to the Gaussian-sigmoid convolution induced by the latent transition yields
\begin{equation}\label{eq:transition}
  \int \mathcal{N}(x_{t+1} \mid x_t, v)\, \sigma(\alpha_{t+1} x_{t+1})\, dx_{t+1} \;\approx\; \sigma\!\left( \frac{\alpha_{t+1}\, x_t}{\sqrt{1 + \phi v \alpha_{t+1}^2}} \right).
\end{equation}
The right-hand side of Eq.~\eqref{eq:recursion} is then the product of two sigmoids: the local likelihood $\sigma(\beta_t x_t)$ and the propagated future message. We project this product back onto the ansatz~\eqref{eq:ansatz} by matching the linear coefficient of $x_t$ at the origin, giving the recursion
\begin{equation}\label{eq:alpha-recursion}
  \alpha_t \;=\; \beta_t \;+\; \frac{\alpha_{t+1}}{\sqrt{1 + \phi v \alpha_{t+1}^2}}.
\end{equation}
Iterating Eq.~\eqref{eq:alpha-recursion} backward from the terminal value $\alpha_T = \beta_T$ determines $\alpha_t$ at all earlier steps. Two structural consequences follow. First, observation noise $s$ enters $\beta_t$ and dampens the per-step contribution to the backward precision. Second, process noise $v$ dampens the propagated precision $\alpha_{t+1}$ by the factor $\sqrt{1 + \phi v \alpha_{t+1}^2}$, so volatility acts as a dual discount factor: it reduces the effective planning horizon by attenuating the backward flow of precision through highly volatile transitions, independently of the geometric discount $\gamma$.

\subsection{Closed-form infinite-horizon precision}

To obtain a closed form, we linearize the recursion by replacing the dynamic denominator by its trapezoidal-rule average over the horizon:
\begin{equation}\label{eq:D}
  D = \tfrac{1}{2}\left(1 + \sqrt{1 + \phi v (\alpha^*)^2}\right),
\end{equation}
where $\alpha^*=\max_t \alpha_t$. The recursion linearizes to $\alpha_t \approx \beta_t + \alpha_{t+1}/D$, which is geometric and unrolls to $\alpha_1 = \sum_{k=1}^{T} \beta_k / D^{k-1}$. Therefore, we have
\begin{equation}
  S \;\equiv\; \lim_{T \to \infty} \alpha_1 \;\approx\; \sum_{k=1}^{\infty} \frac{1}{\sqrt{1 + \phi s\, \gamma^{2(k-1)}}} \left(\frac{\gamma}{D}\right)^{k-1}.
\end{equation}
The summand is slowly varying for $\gamma$ near 1, so the sum is well approximated by the corresponding integral. Substituting $u = (\gamma/D)^{x-1}$ and replacing $\gamma^{2(x-1)}$ by $u^2$, which is exact at $D = 1$ and approximate otherwise, gives
\begin{equation}\label{eq:S_integral}
  S \approx \frac{1}{\ln(D/\gamma)} \int_0^1 \frac{du}{\sqrt{1 + \phi s\, u^2}}.
\end{equation}
This integral has a closed form in terms of the inverse hyperbolic sine, but a simpler algebraic approximation suffices for our purposes. We approximate the integrand's denominator by its trapezoidal-rule average over $[0, 1]$, $\frac{1}{2}(1 + \sqrt{1+\phi s})$, accurate since the denominator is nearly linear on this interval. We further use $-\ln\gamma \approx 1 - \gamma$, valid for $\gamma$ near 1, so that $\ln(D/\gamma) \approx \ln D + (1 - \gamma)$. Substituting into Eq.~\eqref{eq:S_integral} gives:
\begin{equation}
S \approx \frac{2}{(\ln D + 1 - \gamma)(1 + \sqrt{1+\phi s})}. \label{eq:S-closed}
\end{equation}

Evaluating $D$ in turn requires $\alpha^*$. Because $\beta_t$ grows as $t$ decreases, $\alpha_t$ accumulates monotonically toward the present, so $\alpha^* = \alpha_1$ and self-consistency would require $\alpha^* = S$, leading to a fixed-point equation. We avoid solving it by substituting an upper bound instead: volatility attenuates $\alpha_t$ through the damping factor in Eq.~\eqref{eq:alpha-recursion}, so the undamped ($v = 0$) recursion provides an upper bound. At $v = 0$ we have $D = 1$, and Eq.~\eqref{eq:S-closed} reduces to:
\begin{equation}\label{eq:alpha_star}
  \alpha^* \approx \frac{2}{(1 - \gamma)(1 + \sqrt{1+\phi s})}.
\end{equation}
Substituting Eq.~\eqref{eq:alpha_star} into the definition of $D$ closes the form. 

\subsection{The CAUSE index}

At decision time, the predictive belief over the latent state is $p(x_t \mid r_{1:t-1}) = \mathcal{N}(m_{t-1}, \Sigma_t)$ with predictive variance $\Sigma_t = P_{t-1} + v$. Conditioning on infinite-horizon optimality multiplies this prior by the backward message:
\begin{equation}
  p(x_t \mid o_{t:\infty} = 1, r_{1:t-1}) \;\propto\; \mathcal{N}(x_t \mid m_{t-1}, \Sigma_t)\, \sigma(S x_t).
\end{equation}
The product of a Gaussian and a sigmoid is a skew-normal distribution. Its mean is obtained from the exponential-family identity $\mathbb{E}[x_t] = m_{t-1} + \Sigma_t\, \partial_{m_{t-1}} \log Z_t$, where $Z_t$ is the normalizer. The probit approximation gives
\begin{equation} \label{eq:Z}
  Z_t \;\approx\; \sigma\!\left( \frac{S\, m_{t-1}}{\sqrt{1 + \phi \Sigma_t S^2}} \right),
\end{equation}
and differentiating yields
\begin{equation}\label{eq:gated}
  \mathbb{E}[x_t \mid o_{t:\infty} = 1, r_{1:t-1}] \;=\; m_{t-1} + \Sigma_t\, \tilde{\alpha}_t\, \sigma(-\tilde{\alpha}_t\, m_{t-1}),
\end{equation}
with uncertainty-adjusted precision
\begin{equation}
  \tilde{\alpha}_t \;=\; \frac{S}{\sqrt{1 + \phi \Sigma_t S^2}}.
\end{equation}
The sigmoid factor in Eq.~\eqref{eq:gated} depends on $m_{t-1}$ and breaks the additive exploitation-plus-bonus decomposition that the optimal index admits. We replace this gate by a tunable scale parameter $c \in (0, 1)$, recovering an additive $m + B$ form and yielding the CAUSE index:
\begin{equation}\label{eq:cause}
    \lambda_{\mathrm{CAUSE}} \;=\; m_{t-1} + c\, \Sigma_t\, \tilde{\alpha}_t.
\end{equation}
The parameter $c$ controls the overall magnitude of the exploration bonus and does not affect its monotonicities in $s$ or $v$. One principled default emerges from the gate itself: under the symmetric prior and dynamics (Section~\ref{sec:problem}), the marginal distribution of $m_{t-1}$ across histories is symmetric about zero, and the sigmoid satisfies $\sigma(-u) + \sigma(u) = 1$, so the prior-averaged value of the gate is $\tfrac{1}{2}$. We fix $c = \tfrac{1}{2}$ throughout this paper.

%% file: Appendix_C.tex
\newpage
\section{Experimental setup}
\label{app:setup}

This appendix provides implementation details for the experiments of Section~\ref{sec:experiments}. The scale parameter $c$ (Eq. \ref{eq:index_closed}) was held at 0.5 across all experiments and not tuned per-condition; reported regret reflects this fixed setting.

\subsection{Baseline policies}
\label{app:setup_baselines}

All baselines use the same Kalman tracker as CAUSE, with arm-specific known $(v_k, s_k)$ for filtering. We let $m_k$ and $P_k$ denote the agent's posterior mean and variance for arm $k$ at the current decision step.

\paragraph{Myopic.}
Selects $\arg\max_k m_k$ at each step. No exploration bonus.

\paragraph{Thompson sampling \citep{thompson1933likelihood}.}
At each step, samples $\tilde{x}_k \sim \mathcal{N}(m_k, P_k + v_k)$ from each arm's predictive distribution over the next latent state and selects $\arg\max_k \tilde{x}_k$. The variance $P_k + v_k$ is the posterior uncertainty over the arm's mean reward at the next step, including the contribution of latent drift.

\paragraph{Predictive sampling \citep{liu2023nonstationary}.}
Predictive sampling differs from Thompson sampling in its learning target: rather than sampling from the posterior over the current latent state, it samples from a distribution that deprioritizes information whose value will decay due to drift. Proposition~3 of \citet{liu2023nonstationary} reduces the AR(1) Gaussian case to Gaussian sampling with the same predictive mean $m_k$ but a shrunken sampling variance $\tilde{\sigma}^2_k$. Specializing their formula to the random-walk dynamics ($\gamma_{\mathrm{AR}} = 1$) used in this paper gives
\begin{equation}\label{eq:ps_variance}
  \tilde{\sigma}^2_k = \frac{(P_k + v_k)^2}{(P_k + v_k) + x^*_k},
  \qquad
  x^*_k = \tfrac{1}{2}\bigl(v_k + \sqrt{v_k^2 + 4\, v_k\, s_k}\bigr),
\end{equation}
and the agent samples $\tilde{x}_k \sim \mathcal{N}(m_k, \tilde{\sigma}^2_k)$ and selects $\arg\max_k \tilde{x}_k$. The shrinkage factor $x^*_k$ is increasing in both $v_k$ and $s_k$, so $\tilde{\sigma}^2_k \le P_k + v_k$ with equality only in the stationary limit $v_k \to 0$, in which predictive sampling and Thompson sampling coincide \citep[Theorem~1]{liu2023nonstationary}.

\citet{liu2023nonstationary} state their AR(1) results under the assumption $\gamma_{\mathrm{AR}} < 1$, ensuring the latent process admits a steady-state distribution. The random-walk dynamics studied here correspond to the $\gamma_{\mathrm{AR}} \to 1$ limit, in which Eq.~\eqref{eq:ps_variance} remains well-defined as the natural extension of their formula, although the accompanying Bayesian regret bound is not directly applicable in this limit.

\paragraph{UCB \citep{auer2002finite, kaufmann2012bayesian}.}
Selects $\arg\max_k \bigl[m_k + c\sqrt{P_k + v_k}\bigr]$ with $c = 2$. This is a Bayesian adaptation of UCB in which the posterior standard deviation over the next-step latent state plays the role of the confidence radius; with $c = 2$, the bonus corresponds to the upper bound of an approximately 97.7\% one-sided Gaussian credible interval. The bonus depends on $P_k + v_k$ but not on $s_k$, so this UCB variant cannot differentiate informative drift from uninformative observation noise. Note that windowed and discounted variants of UCB and Thompson sampling are not included as baselines because every policy here uses the exact Kalman tracker with known $(v, s)$, which performs the Bayes-optimal forgetting that windows and discount factors approximate when the dynamics are unknown.

\paragraph{IDS \citep{russo2014learning}.}
IDS selects actions by minimizing the information ratio, the squared expected instantaneous regret divided by the expected information gain \citep{russo2014learning}. It requires a Monte Carlo estimate and an optimization at every step, and it yields no per-arm index.

At each step we draw $1000$ samples from the predictive beliefs $\mathcal{N}(m_k,\, \Sigma_k)$ with $\Sigma_k = P_k + v_k$, and estimate the expected instantaneous regret of arm $k$ as $\Delta_k = \mathbb{E}[\max_j x_j] - m_k$. The information gain is the mutual information between the next observation of arm $k$ and its latent state,
\begin{equation}\label{eq:ids_gain}
  g_k = \tfrac{1}{2}\log\!\left(1 + \frac{\Sigma_k}{s_k}\right),
\end{equation}
which increases with the predictive variance and decreases with stochasticity. Following \citet{russo2014learning}, the optimal information-ratio policy has support on at most two actions, so we minimize $(q\Delta_i + (1-q)\Delta_j)^2 / (q g_i + (1-q) g_j)$ over all pairs $(i, j)$ and over $q$ on a grid of 21 points in $[0,1]$, and sample from the resulting two-point distribution.

\paragraph{Oracle.}
Has access to the true latent states $x_k$ and selects $\arg\max_k x_k$ at each step. Provides a lower bound on regret.

\subsection{Computing the optimal Gittins bonus}
\label{app:setup_gittins}
The optimal Gittins bonus is computed by solving the retirement problem of Section~\ref{sec:optimal} numerically, via value iteration with bisection on the retirement salary $\lambda$.

\paragraph{Value iteration.}
For a fixed $\lambda$, the value function $V(m, P)$ satisfies the Bellman equation
\begin{equation}
  V(m, P) = \max\left\{\frac{\lambda}{1 - \gamma},\; m + \gamma\, \mathbb{E}_{m'}[V(m', P')]\right\},
\end{equation}
where $(m', P')$ are the posterior moments after one Kalman update. The next-step posterior variance $P' = (P + v) s / (P + v + s)$ is computed in closed form via the Kalman recursion. The expectation over $m'$ is evaluated via Gauss--Hermite quadrature with 15 nodes, with $m'$ centered at $m$ and standard deviation $(P + v)/\sqrt{P + v + s}$. Value iteration is run on a grid of 251 evenly-spaced $m$ values spanning $[-m_{\text{range}}, m_{\text{range}}]$, with $m_{\text{range}} = 6 \sqrt{P_{\max} + v} / \sqrt{1 - \gamma}$, where $P_{\max}$ is the largest value in the $P$-grid. The $P$-grid is logarithmically spaced from $0.01$ to $\max\{50 + 10 H v_{\max},\; 5(v_{\max} + s_{\max})\}$ with 30 points, where $H = 1/(1-\gamma)$ is the effective horizon. The constant $50$ is set conservatively above the bandit's prior variance ($P_0 = 25$) to ensure the grid spans the simulation's posterior range. Iteration terminates when the maximum value-function change between iterations falls below $10^{-4}$, or after 300 iterations.

\paragraph{Bisection on $\lambda$.}
For each value of $P$, the bonus is the smallest $\lambda$ such that immediate retirement at $m = 0$ is optimal: $V(0, P) \le \lambda / (1 - \gamma)$. We find this $\lambda$ via bisection on the interval $[0, \lambda_{\max}]$, with $\lambda_{\max} = 10 \sqrt{P_{\max} + v} / \sqrt{1 - \gamma}$. Bisection runs for 20 iterations or until the bracket width falls below $10^{-3}$. The reported bonus is the midpoint of the final bracket.

\paragraph{Bonus comparison figure.}
For the bonus shape comparison (Section~\ref{sec:bonus_gittins}), which evaluates the bonus only at fixed $P_{\text{ref}}$, we use a tighter local $P$-grid (linearly spaced from $0.5 P_{\text{ref}}$ to $2 P_{\text{ref}}$ with 20 points), 121 $m$-grid points, 11 quadrature nodes, and at most 200 value-iteration steps; the other tolerances are unchanged. The reduced settings preserve shape accuracy at $P_{\text{ref}}$ while reducing computation per $(v, s)$ point.

\subsection{Computing the Whittle index}
\label{app:setup_whittle}

The Whittle index is computed by value iteration on the single-arm subsidy problem \citep{whittle1988restless}. At each step the agent either activates the arm, receiving reward and updating $m$ and $P$, or rests it, receiving a subsidy $\lambda$ while the arm continues to drift, so that $m$ is unchanged and $P \to P + v$. The index is the subsidy at which the agent is indifferent between the two actions, and the same shift-coupling as in Proposition~\ref{prop:decomposition} gives $W(m, P) = m + w(P)$, so we tabulate $w(P)$.

The numerical scheme follows Appendix~\ref{app:setup_gittins}, with the same $m$- and $P$-grids, Gauss--Hermite quadrature, and bisection tolerances. Because resting pushes $P$ upward by $v$ at every step, value iteration is run on a $P$-grid extended four-fold above the reported grid with four additional logarithmically spaced nodes; without this headroom the clamped interpolation at the top of the grid makes resting appear artificially costly and inflates $w(P)$ at large $P$. Each bisection step is warm-started from the previous value function shifted by $(\lambda_{\text{new}} - \lambda_{\text{old}})/(1-\gamma)$, and the iteration limit is raised to 600, since the two-action problem converges at rate $\gamma$ and has no absorbing branch. 

\subsection{Computing the optimal restless policy at $K=2$}
\label{app:optk2}
At $K=2$ the optimal restless policy is computed by value iteration on the joint belief state, which reduces to three dimensions: the mean difference $m_1 - m_2$ and the two posterior variances $P_1$ and $P_2$. The numerical scheme follows Appendices~\ref{app:setup_gittins} and~\ref{app:setup_whittle}, with the
same logarithmically spaced $P$-grid construction, the same Gauss--Hermite quadrature with $15$ nodes,
and the same $P$-grid headroom as in Appendix~\ref{app:setup_whittle}, since resting here likewise
pushes $P$ upward at every step. Notably, here the state is three-dimensional rather than two, so the
grid carries $20$ points per $P$ axis and interpolation is trilinear; correspondingly, the tolerance is
relaxed to $10^{-3}$ and the iteration limit raised to $1500$.

The six configurations, given as $(v_1, v_2; s_1, s_2)$, span the three regimes of Section~\ref{sec:regret} at two settings each. The mixed configurations, $(1, 4; 25, 9)$ and $(1, 4; 9, 25)$, assign the two arms opposite corners of the mixed $(v, s)$ grid, and are therefore the two ways of pairing high volatility with low stochasticity or with high. The remaining four vary one axis while holding the other fixed: $(1, 1; 9, 25)$ and $(4, 4; 9, 900)$ for $s$-dominant, and $(1, 4; 9, 9)$ and $(1, 100; 25, 25)$ for $v$-dominant. Other settings ($T = 200$, $\gamma = 0.95$, $P_0 = 25$, $1000$ Monte Carlo runs) match Section~\ref{sec:regret}. Finally, we verify the implementation at $v_1 = v_2 = 0$, where the Gittins index is provably
optimal and the simulated regret is tied with the Gittins-per-arm baseline.

\subsection{Statistical comparison}
Ties are assessed by a single-step paired max-$t$ bootstrap \citep{westfall1993resampling} at $\alpha = 0.05$. Since all policies are run on the same seeds, runs are resampled jointly, preserving the pairing; the critical value is the $95$th percentile, over $5000$ resamples, of the maximum across all ordered policy pairs of the centred paired $t$-statistic $(D^*_{jk} - D_{jk})/\mathrm{se}^*_{jk}$. A policy is reported as tied with the best if its $t$-statistic against the best does not exceed this value. Taking the maximum over all pairs rather than only those involving the best controls the family-wise error rate and accounts for the best policy having been selected from the same data.

\subsection{Lesion analysis}
\label{app:setup_lesion}
The joint inference of $(\hat{v}, \hat{s})$, the agent's running estimates of volatility and stochasticity, uses the framework of \citet{piray2021model}, with each agent differing only in two sensitivity parameters $\lambda_v, \lambda_s \in [0, 1]$ that control the update rate of the corresponding estimate. Following \citet{piray2021model}, we set $\lambda_v = \lambda_s = 0.1$ for the healthy agent. Blind agents are defined by setting the corresponding parameter to $0$, removing the agent's ability to attribute experienced noise to that source and pinning the corresponding estimate ($\hat{s}$ or $\hat{v}$) at its initial value. Initial values for both $\hat{v}$ and $\hat{s}$ are set to the midpoint of the true values, following \citet{piray2021model}.

Each agent acts on its own inferred $(\hat{s}, \hat{v})$ and its own evolving posterior. The three agents are run in the $s$-dominant and $v$-dominant regimes described below, each of which varies one noise parameter across arms while holding the other fixed, so that each regime isolates one axis of inference. Other settings ($K = 4$, $T = 200$, $\gamma = 0.95$) match Section~\ref{sec:regret}, and all agents are run on the same 500 seeds, so that contrasts between them are paired. The rate of exploratory choices is the fraction of trials on which the agent selects an arm other than the one with the highest posterior mean, reported separately for the high- and low-noise arms of each regime. Contrasts against the healthy agent are reported with 95\% confidence intervals. The framework's reference implementation \cite{piray2021model} is available under the MIT license.

\subsection{Bandit configuration}
\label{app:setup_bandit}

All experiments use Gaussian state-space bandits with latent dynamics $x_{t+1} \mid x_t \sim \mathcal{N}(x_t, v)$ and observation model $r_t \mid x_t \sim \mathcal{N}(x_t, s)$, with arm-specific noise parameters $(v_k, s_k)$. Latent states are initialized as $x_0 \sim \mathcal{N}(0, P_0)$ independently for each arm, with $P_0 = 25$. Each arm's true initial mean is therefore drawn from the agent's prior, providing nontrivial reward differentiation across arms regardless of volatility. The discount factor in all main analyses is $\gamma = 0.95$. Sensitivity to $\gamma$ is reported in Appendix~\ref{app:robustness}.

\paragraph{Heterogeneous-arms regret experiment.}
The experiments of Section~\ref{sec:regret} use $K = 4$ arms over $T = 200$ timesteps in three regimes. The mixed regime partitions arms equally across the four cells of the $(v, s)$ grid with $v_{\text{low}} = 1$, $v_{\text{high}} = 4$, $s_{\text{low}} = 9$, $s_{\text{high}} = 25$. The $s$-dominant regime fixes $v = 4$ across all arms, with $s = 9$ for two arms and $s = 900$ for two arms. The $v$-dominant regime fixes $s = 25$ across all arms, with $v = 1$ for two arms and $v = 100$ for two arms. Sensitivity to the $(v, s)$ scale, the discount factor $\gamma$, and the number of arms $K$ is reported in Appendix~\ref{app:robustness}.

\paragraph{Bonus comparison experiment.}
The experiment of Section~\ref{sec:bonus_gittins} sweeps one noise parameter while holding the other fixed at the high value of the mixed regime ($v = 4$ when sweeping $s$; $s = 25$ when sweeping $v$). Each axis is swept over the range $[10, 1000]$ at 14 logarithmically spaced points. The posterior variance is held at $P_{\text{ref}}$, set separately for each axis to the median, across the swept range, of the stationary posterior variance $P_{\infty}(v, s) = (\sqrt{v^2 + 4vs} - v)/2$ implied by the Kalman recursion.

\paragraph{Rested verification at $v=0$.}
\label{app:v0_verification}
The verification of Section~\ref{sec:rested_verification} sets $v = 0$ for all $K = 4$ arms and tests two stochasticity configurations: $s \in \{9, 25\}$ (the $s$ values from the mixed regime) and $s \in \{9, 900\}$ (from the s-dominant regime), with two arms at each $s$ level. Other settings ($T = 200$, $\gamma = 0.95$, $1000$ Monte Carlo runs) match Section~\ref{sec:regret}.

\paragraph{Monte Carlo procedure.}
All experiments are averaged over 1000 independent runs with different random seeds. Reported quantities are mean $\pm$ standard error of the mean (SEM).

\subsection{Compute resources}
\label{app:compute}

All experiments were run on CPUs. Each independent run completes in seconds, and full experiments (1000 Monte Carlo runs across all conditions) complete in a few minutes per experiment on a standard laptop. Total compute for all experiments and preliminary explorations is well under one CPU core day.

%% file: Appendix_D.tex
\newpage
\section{Extended main empirical results}
\label{app:supp_results}
This appendix presents supplementary results extending the main empirical claims of Sections~\ref{sec:bonus_gittins} and \ref{sec:rested_verification}.

\subsection{Regret as a function of $K$}
\label{app:k_table}

A more complete version of Table~\ref{tab:regret_K}, with the standard error of the mean included for every entry, is reported in Table~\ref{tab:regret_K_full}.

\begin{table}[H]
  \centering
  \caption{Final cumulative discounted regret across problem sizes
  ($T = 200$, $\gamma = 0.95$, mean and SEM across 1000 Monte Carlo runs).
  Bold marks the best policy and any policy statistically tied with it.}
  \label{tab:regret_K_full}
  \resizebox{\linewidth}{!}{%
  \begin{filecontents*}{figures/regret+SEM_K_table_tex.csv}
K,regime,Myopic,Thompson,Predictive,UCB,IDS,Gittins,Whittle,CAUSE
4,mixed,59.47 $\pm$ 1.69,56.72 $\pm$ 0.72,54.63 $\pm$ 0.71,47.36 $\pm$ 0.73,43.42 $\pm$ 1.02,42.61 $\pm$ 0.89,\textbf{39.25 $\pm$ 0.69},\textbf{39.18 $\pm$ 0.83}
,s-dominant,84.53 $\pm$ 2.52,104.76 $\pm$ 2.36,100.03 $\pm$ 2.28,100.13 $\pm$ 2.76,87.46 $\pm$ 2.50,\textbf{74.06 $\pm$ 1.96},\textbf{74.73 $\pm$ 1.96},\textbf{74.04 $\pm$ 2.30}
,v-dominant,323.29 $\pm$ 9.23,141.78 $\pm$ 1.67,141.37 $\pm$ 1.65,139.48 $\pm$ 2.31,115.37 $\pm$ 2.68,125.02 $\pm$ 2.53,107.07 $\pm$ 1.80,\textbf{103.30 $\pm$ 1.93}
8,mixed,113.02 $\pm$ 2.29,99.80 $\pm$ 0.99,95.68 $\pm$ 0.96,78.91 $\pm$ 1.04,74.95 $\pm$ 1.42,\textbf{71.99 $\pm$ 1.27},\textbf{69.23 $\pm$ 1.07},\textbf{70.13 $\pm$ 1.29}
,s-dominant,127.06 $\pm$ 2.72,169.43 $\pm$ 2.51,154.05 $\pm$ 2.51,142.95 $\pm$ 2.72,113.50 $\pm$ 2.62,\textbf{108.64 $\pm$ 2.16},\textbf{110.59 $\pm$ 2.11},\textbf{106.12 $\pm$ 2.25}
,v-dominant,600.66 $\pm$ 10.86,256.98 $\pm$ 2.39,259.51 $\pm$ 2.33,217.84 $\pm$ 2.08,205.69 $\pm$ 4.19,\textbf{182.94 $\pm$ 2.78},\textbf{178.61 $\pm$ 2.51},\textbf{182.42 $\pm$ 3.04}
12,mixed,151.23 $\pm$ 2.55,131.34 $\pm$ 1.09,128.02 $\pm$ 1.12,102.79 $\pm$ 1.31,100.08 $\pm$ 1.66,\textbf{94.40 $\pm$ 1.50},\textbf{92.15 $\pm$ 1.36},\textbf{94.15 $\pm$ 1.48}
,s-dominant,156.66 $\pm$ 2.75,206.78 $\pm$ 2.76,190.15 $\pm$ 2.49,174.25 $\pm$ 2.79,\textbf{129.41 $\pm$ 2.50},\textbf{131.79 $\pm$ 2.29},135.61 $\pm$ 2.22,\textbf{130.91 $\pm$ 2.41}
,v-dominant,803.44 $\pm$ 10.89,351.33 $\pm$ 2.99,351.22 $\pm$ 2.96,304.44 $\pm$ 2.41,291.95 $\pm$ 5.37,\textbf{256.20 $\pm$ 3.75},\textbf{253.84 $\pm$ 3.27},\textbf{259.59 $\pm$ 3.89}
16,mixed,180.25 $\pm$ 2.66,157.55 $\pm$ 1.22,152.96 $\pm$ 1.26,119.06 $\pm$ 1.42,116.66 $\pm$ 1.98,\textbf{107.46 $\pm$ 1.74},\textbf{111.28 $\pm$ 1.54},\textbf{110.26 $\pm$ 1.74}
,s-dominant,180.02 $\pm$ 2.82,239.60 $\pm$ 2.73,216.24 $\pm$ 2.63,192.62 $\pm$ 2.74,\textbf{144.06 $\pm$ 2.39},\textbf{144.84 $\pm$ 2.25},150.74 $\pm$ 2.21,\textbf{144.04 $\pm$ 2.29}
,v-dominant,962.98 $\pm$ 10.29,432.16 $\pm$ 3.15,433.90 $\pm$ 3.26,374.12 $\pm$ 2.92,357.42 $\pm$ 5.88,\textbf{317.76 $\pm$ 4.34},\textbf{314.20 $\pm$ 3.95},\textbf{321.54 $\pm$ 4.61}
\end{filecontents*}
\csvreader[
    tabular = ll*{8}{c},
    table head = \toprule $K$ & Regime & Myopic & Thompson & Predictive & UCB & IDS & Gittins & Whittle & CAUSE \\ \midrule,
    table foot = \bottomrule,
    late after line = \\\noalign{\ifcsvstrcmp{\regime}{mixed}{\vskip2pt\hrule height0.4pt\vskip2pt}{}},
    late after last line = \\,
    head to column names
  ]{figures/regret+SEM_K_table_tex.csv}{}{\K & \regime & \Myopic & \Thompson & \Predictive & \UCB & \IDS & \Gittins & \Whittle & \CAUSE}}
\end{table}

\subsection{Bonus shape on the volatility axis}
\label{app:bonus_v}
Figure~\ref{fig:bonus_v} extends the bonus comparison of Section~\ref{sec:bonus_gittins} to the volatility axis: each policy's bonus as a function of $v$, with $s$ fixed at $25$ and $P_{\text{ref}}$ as described in Appendix~\ref{app:setup_bandit}.

\begin{figure}[h]
  \centering
  \includegraphics[width=0.5\linewidth]{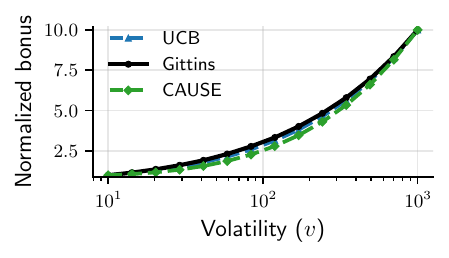}
  \caption{Exploration bonus as a function of volatility $v$ at fixed $s = 25$, normalized to a common range. CAUSE, UCB, and Gittins all show the correct increasing direction; CAUSE lies modestly below Gittins.}
  \label{fig:bonus_v}
\end{figure}

All three policies increase with $v$, as the structural analysis prescribes. CAUSE lies modestly below Gittins, reflecting the additional volatility-aware damping (the $D$-factor; see Appendix~\ref{app:derivation}) that accounts for the diminishing value of distant-future information when latent states drift. UCB tracks Gittins closely along this axis since UCB's $\sqrt{P + v}$ scaling implicitly captures the volatility dependence, even though it does not capture stochasticity dependence (Section~\ref{sec:bonus_gittins}). The differences along the volatility axis are subtle in bonus shape but produce substantial differences in regret in the v-dominant regime (Section~\ref{sec:regret}).

\subsection{Rested verification at extreme stochasticity}
\label{app:v0_verification}
Section~\ref{sec:rested_verification} reports the rested verification ($v = 0$) at moderate stochasticity heterogeneity ($s \in \{9, 25\}$). Figure~\ref{fig:v0_extreme} extends the verification to extreme stochasticity heterogeneity ($s \in \{9, 900\}$), matching the high-$s$ value of the s-dominant regime.

\begin{figure}[H]
  \centering
  \includegraphics[width=0.85\linewidth]{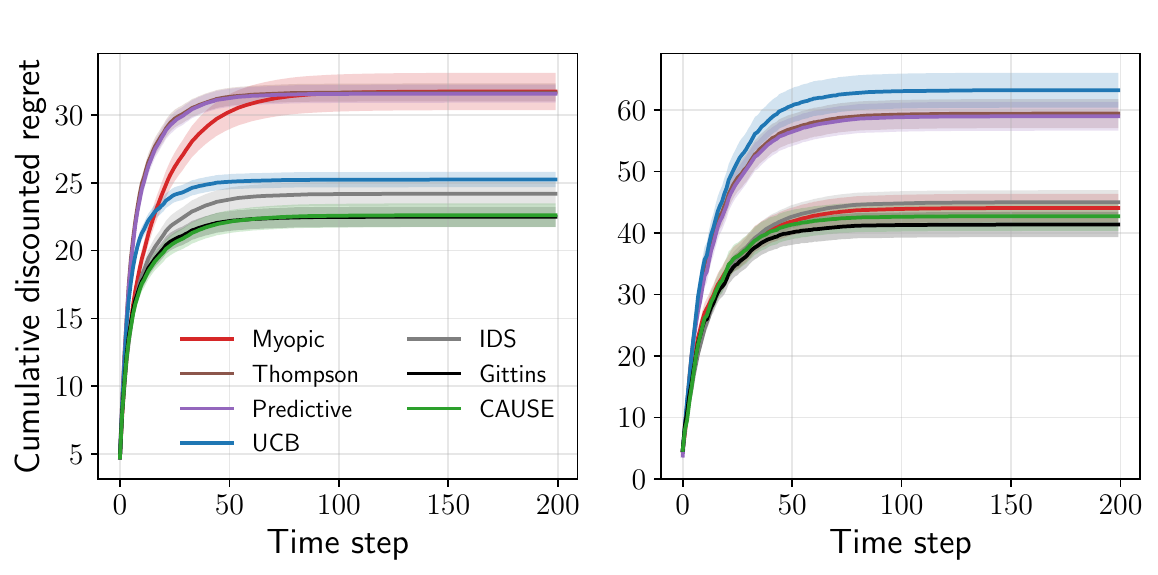}
  \caption{Cumulative discounted regret at $v = 0$ for $K = 4$ arms over $T = 200$ steps, $\gamma = 0.95$, $1000$ Monte Carlo runs. \emph{Left}: $s \in \{9, 25\}$. \emph{Right}: $s \in \{9, 900\}$. CAUSE and Gittins-per-arm overlap within Monte Carlo precision in both configurations.}
  \label{fig:v0_extreme}
\end{figure}

At extreme stochasticity, CAUSE achieves $42.71 \pm 2.40$ and Gittins-per-arm achieves $41.36 \pm 2.04$ (mean $\pm$ SEM, 1000 runs); the difference of $1.35$ is well within sampling error. The closeness of CAUSE to the rested-optimal Gittins reference is therefore preserved across the stochasticity range tested in the main experiments.

\subsection{Isolating the closed-form steps by numerically computing $S$}
\label{app:numS}

The closed form for $S$ (Eq.~\ref{eq:S-closed}) is obtained from the recursion (Eq.~\ref{eq:alpha-recursion}) through four steps: the trapezoidal average $D$ replacing the dynamic denominator, the replacement of the sum by an integral, the trapezoidal average of the integrand together with $-\ln\gamma \approx 1-\gamma$, and the substitution of the undamped limit for $\alpha^*$. To isolate their joint contribution, we evaluate CAUSE-$S_{\text{numerical}}$, a variant in which $S$ is obtained by numerically iterating Eq.~\ref{eq:alpha-recursion} backward to convergence rather than in closed form. The index, the Kalman tracker, and the scale $c$ are otherwise identical, so the comparison removes these four steps and leaves the remaining approximations untouched. Across the three regimes of Section~\ref{sec:regret} ($K=4$, $T=200$, $\gamma=0.95$, 1000 Monte Carlo
runs), CAUSE achieves lower regret than CAUSE-$S_{\text{numerical}}$ in the mixed regime and is
statistically tied with it in the other two (Table~\ref{tab:numS}). The entire closed-form collapse
therefore costs no regret in any regime tested; concretely, the four steps it replaces account for a
difference of at most a few percent, and in the mixed regime that difference favors the closed form.

\begin{table}[h]
  \centering
  \caption{Final cumulative discounted regret for CAUSE and for the variant in which $S$ is obtained by numerically iterating Eq.~\ref{eq:alpha-recursion} ($K = 4$, $T = 200$, $\gamma = 0.95$, 1000 Monte Carlo runs, mean $\pm$ SEM). Bold marks the lower value in each row and any value statistically tied with it under a paired max-$t$ bootstrap at the 95\% level.}
  \label{tab:numS}
  \begin{tabular}{lcc}
    \toprule
    Regime & CAUSE & CAUSE-$S_{\text{numerical}}$ \\
    \midrule
    mixed       & $\mathbf{39.18 \pm 0.83}$  & $42.02 \pm 0.94$ \\
    $s$-dominant  & $\mathbf{74.04 \pm 2.30}$  & $\mathbf{75.97 \pm 2.38}$ \\
    $v$-dominant  & $\mathbf{103.30 \pm 1.93}$ & $\mathbf{105.40 \pm 2.04}$ \\
    \bottomrule
  \end{tabular}
\end{table}

A second variant isolates the volatility damping itself. Setting $D = 1$ in Eq.~\eqref{eq:D} removes the damping of $S$ while leaving the predictive variance $\Sigma_t = P_{t-1} + v$ in the bonus untouched, and costs $3.0$, $3.1$, and $2.2$ regret points across the three regimes. The degradation is consistent but modest, and the ablated agent still retains almost all of CAUSE's advantage over Gittins-per-arm in the v-dominant regime ($105.50$ against $125.02$). This follows from the structure of the index, since $\tilde\alpha_t$ saturates at $1/\sqrt{\phi\,\Sigma_t}$ for large $\Sigma_t$ independently of $S$, so the damping contributes mainly at small to moderate uncertainty. Volatility therefore improves the index through two channels of unequal size: the predictive variance, which carries most of the advantage, and the damping of $S$, which adds a smaller but consistent increment. 

\subsection{Validation methods for each approximation}
This table reports the validation methods that bear on each approximation in the derivation of CAUSE. The two steps without an analytic bound are precisely the two that replace a quantity by a chosen functional form rather than approximating it: the ansatz restricts the backward message to a one-parameter family, with no control on the residual, whereas the constant gate replaces a factor ranging over $(0,1)$ by a scalar, for which the only available bound is vacuous. Both are therefore validated empirically rather than analytically, through the rested optimum at $v = 0$ and the restless optimum at $K = 2$. The replacement of the gate by a constant is additionally supported by a sweep over $c$ (Appendix~\ref{app:c}).

\begin{table}[h]
  \centering
  \setlength{\tabcolsep}{4pt}
  \caption{The nine approximations of Appendix~\ref{app:derivation} and their corresponding validation methods. A step is covered by the rested optimum when its error does not vanish at $v = 0$, and by CAUSE-$S_{\text{numerical}}$ when that variant computes it differently, so that the regret difference of Appendix~\ref{app:numS} bounds its contribution. The restless optimum at $K = 2$ covers every step and is therefore the only method exercising the full stack.}
  \label{tab:approx_coverage}
  \begin{tabular}{llcccc}
    \toprule
    & & \multicolumn{4}{c}{Validation method} \\
    \cmidrule(l){3-6}
    Approximation & Eq. & \makecell{Rested\\optimum} & \makecell{CAUSE-\\$S_{\text{num.}}$} & \makecell{Restless\\optimum} & \makecell{Analytic\\bound} \\
    \midrule
    Reward-marginalization probit & \eqref{eq:local} & yes & no & yes & yes \\
    Sigmoid ansatz & \eqref{eq:ansatz}, \eqref{eq:alpha-recursion} & yes & no & yes & no \\
    Transition probit & \eqref{eq:transition} & no & no & yes & yes \\
    Trapezoidal average $D$ & \eqref{eq:D} & no & yes & yes & yes \\
    Sum replaced by integral & \eqref{eq:S_integral} & yes & yes & yes & yes \\
    $-\ln\gamma \approx 1-\gamma$ & \eqref{eq:S-closed} & yes & yes & yes & yes \\
    $\alpha^*$ at the undamped limit & \eqref{eq:alpha_star} & no & yes & yes & yes \\
    Normalizer probit & \eqref{eq:Z} & yes & no & yes & yes \\
    Constant gate $c$ & \eqref{eq:cause} & yes & no & yes & no \\
    \bottomrule
  \end{tabular}
\end{table}

\subsection{Lesion analysis at fixed variance}
\label{app:lesion_fixedP}

Section~\ref{sec:lesion} reports the lesioned agents in closed loop, where realized exploration reflects both the inferred $(\hat s, \hat v)$ and the posterior variance that stochasticity itself drives. Here we hold the posterior variance fixed at a reference value $P_{\text{ref}}$, the median across the four arms of the stationary posterior variance, to isolate the first channel. The three agents of Section~\ref{sec:lesion} each observe 200 trials of reward data in every cell of the mixed $(v, s)$ grid, and their final estimates are used to evaluate the CAUSE bonus at $P_{\text{ref}}$, so that differences across agents reflect their inference alone.

\begin{figure}[h]
  \centering
  \includegraphics[width=0.8\linewidth]{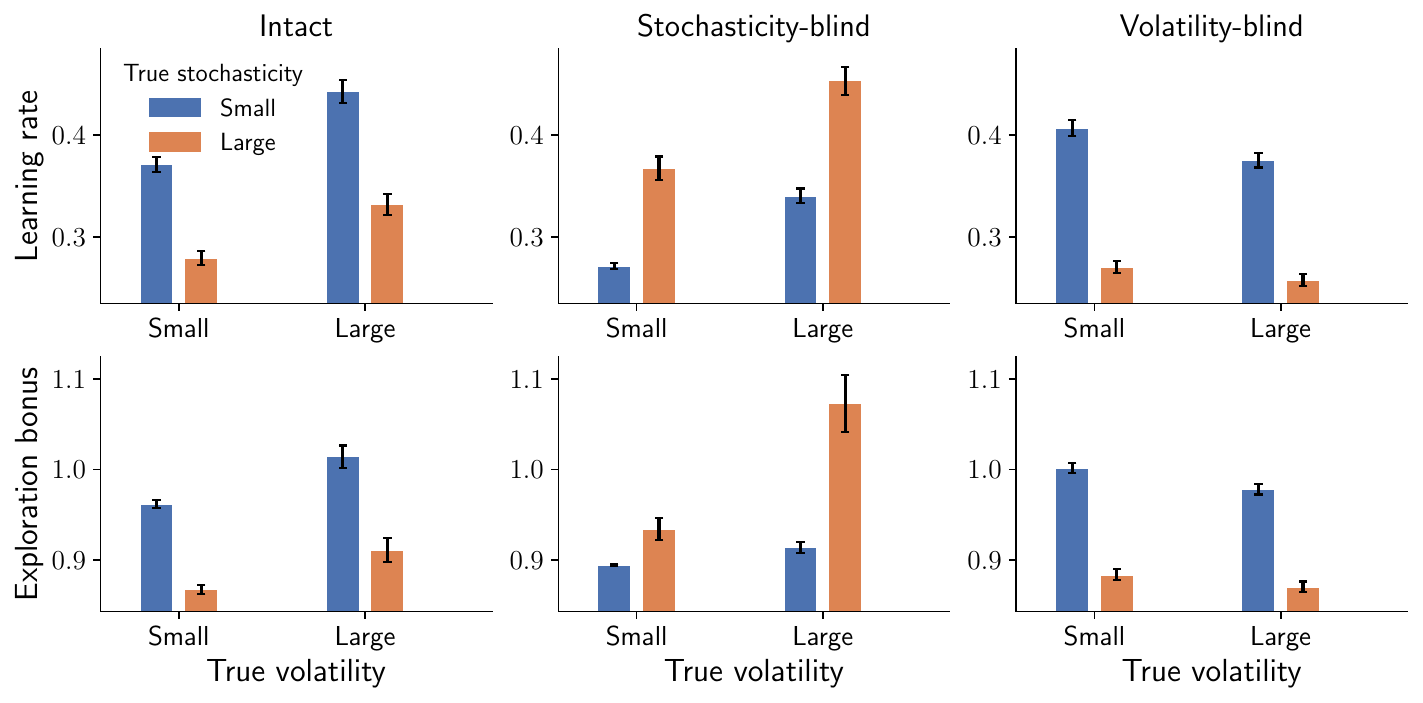}
  \caption{Learning rate (top row) and CAUSE exploration bonus (bottom row) for the intact, stochasticity-blind, and volatility-blind agents at fixed posterior variance, averaged across seeds. Each lesion reverses its corresponding monotonicity: stochasticity-blindness the $s$-monotonicity, volatility-blindness the $v$-monotonicity.}
  \label{fig:lesion}
\end{figure}

Under this restriction each lesion reverses its own monotonicity (Figure~\ref{fig:lesion}). The stochasticity-blind agent's bonus increases in $s$: the variance it cannot attribute to $\hat s$ inflates $\hat v$ instead, raising the bonus through the channel that registers genuine drift in a healthy agent. The volatility-blind agent is the symmetric case, its bonus decreasing in $v$ as the variance is absorbed into $\hat s$. Each lesion leaves the orthogonal axis intact, and the learning rate reverses in the same way, as established previously \citep{piray2021model, piray2024computational}.\\

%% file: Appendix_E.tex
\newpage
\section{Robustness}
\label{app:robustness}

\subsection{Discount factor}
\label{app:gamma}
We replicate the mixed-regime experiment at three additional discount factors,
$\gamma \in \{0.8, 0.9, 0.98\}$. Across all three, CAUSE achieves the lowest regret or is statistically tied with the best.

\begin{figure}[H]
  \centering
  \includegraphics[width=\linewidth]{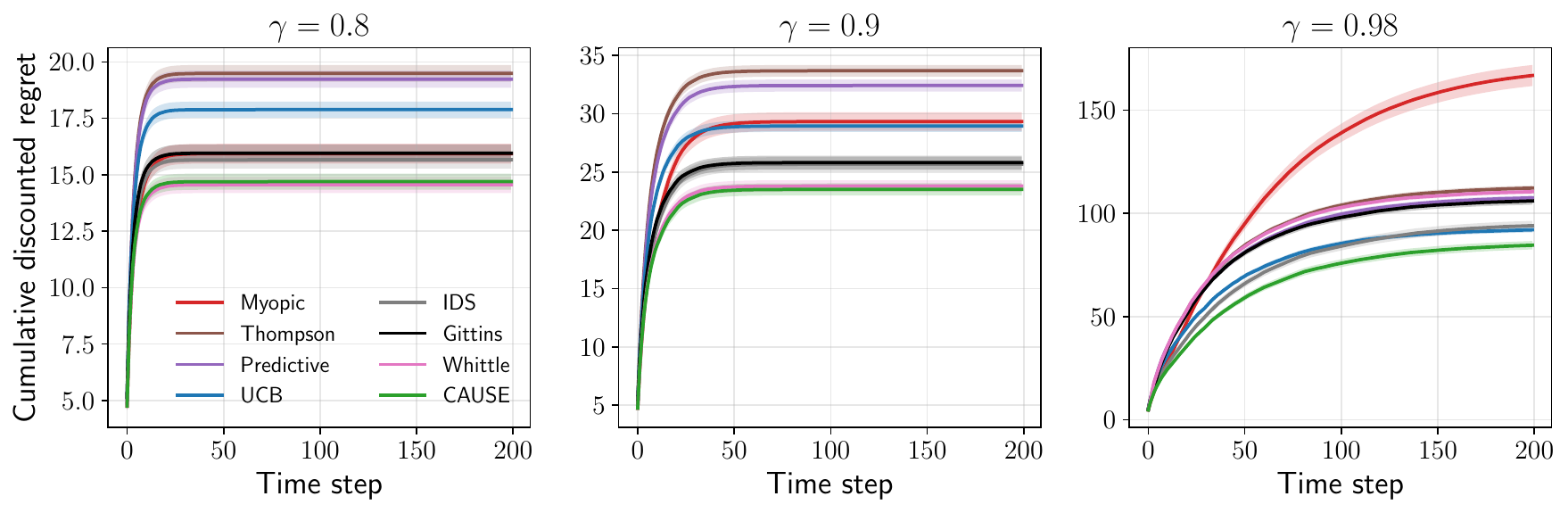}
  \caption{Cumulative discounted regret in the mixed regime across discount factors $\gamma \in \{0.8, 0.9, 0.98\}$ ($K = 4$, $T = 200$, $1000$ Monte Carlo runs).}
  \label{fig:gamma_sweep}
\end{figure}

\subsection{Exploration scale $c$}
\label{app:c}
The scale parameter $c$ of Eq.~\eqref{eq:index_closed} is fixed at $0.5$ throughout, the prior-averaged value of the gate it replaces (Appendix~\ref{app:derivation}). To check that this derivation-fixed value is not merely convenient, we sweep $c \in \{0.25, 0.5, 0.75, 1.0\}$ across the three regimes of Section~\ref{sec:regret}.

\begin{table}[h]
  \centering
  \caption{Final cumulative discounted regret for CAUSE across values of the scale parameter $c$ ($K = 4$, $T = 200$, $\gamma = 0.95$, 1000 Monte Carlo runs, mean $\pm$ SEM).}
  \label{tab:c_sweep}
  \begin{tabular}{lccc}
    \toprule
    $c$ & mixed & $s$-dominant & $v$-dominant \\
    \midrule
    $0.25$        & $44.66 \pm 1.10$ & $75.04 \pm 2.38$ & $131.94 \pm 3.67$ \\
    $0.50$ (used)  & $39.18 \pm 0.83$ & $74.04 \pm 2.30$ & $103.30 \pm 1.93$ \\
    $0.75$        & $38.93 \pm 0.70$ & $75.56 \pm 2.33$ & $107.12 \pm 1.70$ \\
    $1.00$         & $41.30 \pm 0.63$ & $81.15 \pm 2.36$ & $121.05 \pm 1.89$ \\
    \bottomrule
  \end{tabular}
\end{table}

The derivation-fixed $c = 1/2$ is best or statistically tied with the best in every regime, and the optimum is shallow: every value tested beats UCB, Thompson sampling, and predictive sampling in all three regimes. Performance degrades at $c = 1$, the no-gate limit, so the gate term contributes beyond an overall rescaling. This contrasts with UCB, whose best exploration constant varies by regime with no single good value (Appendix~\ref{app:ucb_c}).

\subsection{UCB exploration constant}
\label{app:ucb_c}
The UCB baseline in Section~\ref{sec:regret} uses the canonical $c = 2$ scaling, corresponding to the standard $2\sigma$ upper confidence bound of Bayesian UCB \citep{kaufmann2012bayesian}. UCB allocates exploration in proportion to $c \sqrt{P + v}$, scaling with predictive variance and not with stochasticity $s$. This structural insensitivity to $s$ means that no choice of $c$ can produce the s-axis discounting that the optimal exploration bonus exhibits (Sections~\ref{sec:optimal} and \ref{sec:bonus_gittins}); UCB therefore over-explores high-$s$ arms regardless of the global exploration scale.

We verify this structurally by sweeping $c \in \{0.5, 1, 2, 3\}$ across the three regimes of Section~\ref{sec:regret}. Two patterns emerge. First, the optimal $c$ for UCB varies across regimes: the best-performing value is smaller in s-dominant ($c = 0.5$) than in v-dominant or mixed regimes ($c = 1$). A single fixed $c$ therefore cannot be globally optimal for UCB. Second, tuning $c$ is enough to match CAUSE in the mixed and v-dominant regimes, where UCB at $c = 1$ is statistically tied with it, but not in s-dominant, where UCB remains behind at every value tested (Table~\ref{tab:ucb_c}). This is the regime in which insensitivity to $s$ should harm more, and no choice of $c$ overcomes it. The structural insensitivity is visible directly in the bonus comparison (Figure~\ref{fig:bonus}): UCB’s bonus is exactly flat as a function of $s$, while CAUSE and Gittins decrease monotonically. This is a structural feature of UCB’s allocation rule and cannot be repaired by rescaling $c$.

\begin{table}[h]
  \centering
  \caption{UCB regret across values of its constant $c$, with the myopic baseline and CAUSE for reference ($K = 4$, $T = 200$, $\gamma = 0.95$, 1000 Monte Carlo runs, mean $\pm$ SEM). Bold marks the best policy in each regime and any policy statistically tied with it.}
  \label{tab:ucb_c}
  \begin{tabular}{lccc}
    \toprule
    Policy & mixed & $s$-dominant & $v$-dominant \\
    \midrule
    Myopic          & $59.47 \pm 1.69$          & $84.53 \pm 2.52$          & $323.29 \pm 9.23$ \\
    UCB, $c = 0.5$  & $44.35 \pm 1.09$          & $\mathbf{78.56 \pm 2.51}$ & $118.81 \pm 3.02$ \\
    UCB, $c = 1$    & $\mathbf{41.43 \pm 0.87}$ & $82.22 \pm 2.57$          & $\mathbf{104.72 \pm 1.89}$ \\
    UCB, $c = 2$    & $47.36 \pm 0.73$          & $100.13 \pm 2.76$         & $139.48 \pm 2.31$ \\
    UCB, $c = 3$    & $59.46 \pm 0.78$          & $119.83 \pm 3.08$         & $182.04 \pm 2.97$ \\
    \midrule
    CAUSE           & $\mathbf{39.18 \pm 0.83}$ & $\mathbf{74.04 \pm 2.30}$ & $\mathbf{103.30 \pm 1.93}$ \\
    \bottomrule
  \end{tabular}
\end{table}

%% file: Appendix_F.tex
\newpage
\section{Control-as-inference with an exponential optimality likelihood}
\label{app:exp}

We build CAUSE on an optimality likelihood given by the logistic sigmoid of the discounted reward (Eq.~\ref{eq:opt_likelihood}). Some major applications of control-as-inference instead define optimality as an exponential function of reward \citep{levine2018reinforcement, todorov2008general}, which is unnormalized and unbounded (it exceeds one whenever $r > 0$), but which yields recursions that are exactly conjugate under Gaussian dynamics. We therefore derive that alternative in parallel here. Because those recursions are conjugate, the derivation is more straightforward and avoids the sigmoid ansatz of Appendix~\ref{app:derivation} entirely; it nonetheless fails, as shown below.

We retain the state process of Section~\ref{sec:problem},
\begin{equation}\label{eq:ef-dynamics}
    p(x_{t+1}\mid x_t) = \mathcal{N}(x_{t+1}\mid x_t,\, v),
    \qquad
    p(r_t\mid x_t) = \mathcal{N}(r_t\mid x_t,\, s),
\end{equation}
and define optimality using the exponential form:
\begin{equation}\label{eq:ef-likelihood}
    p(o_t{=}1\mid r_t) = \exp(\gamma^{\,t-1} r_t).
\end{equation}
The likelihood in Eq.~\eqref{eq:ef-likelihood} exceeds one whenever $r_t > 0$ and is therefore not a proper distribution. This is the standard formulation of control-as-inference \citep{levine2018reinforcement}, and it is harmless for the inference itself. The unobserved reward is then marginalized against the observation model. Here the marginalization is the Gaussian moment generating function and is available in closed form, with no probit approximation required:
\begin{equation}\label{eq:ef-local}
    p(o_t{=}1\mid x_t)   
    = \exp\!\left(\gamma^{\,t-1} x_t + \tfrac{1}{2}\gamma^{\,2(t-1)} s\right).
\end{equation}
The contrast with Eq.~\eqref{eq:local} is already visible. There, stochasticity enters the coefficient of $x_t$ directly and damps the per-step contribution to the backward precision. Here it contributes a term that is constant in $x_t$, which therefore accumulates into the normalizer and cancels in the posterior. 

The backward message $b_t(x_t) \equiv p(o_{t:T}{=}1 \mid x_t)$ satisfies the recursion of Eq.~\eqref{eq:recursion}. Since the exponential form is closed under both multiplication by Eq.~\eqref{eq:ef-local} and convolution with the Gaussian transition, it is easy to see that $b_t(x_t) = \exp(S_t x_t + E_t)$, with $S_t = \sum_{k=t}^{T}\gamma^{\,k-1}$ and $E_t$ an additive constant independent of $x_t$. Since $\gamma < 1$, the series converges:
\begin{equation}\label{eq:ef-S}
    S_t = \sum_{k=t}^{\infty}\gamma^{\,k-1} = \frac{\gamma^{\,t-1}}{1-\gamma} .
\end{equation}
The constant $E_t$ converges as well, accumulating the $s$- and $v$-dependent terms contributed at each step, but it cancels in the posterior. Conditioning the predictive belief on future optimality and completing the square gives
\begin{equation}\label{eq:ef-posterior}
    p(x_t \mid o_{t:\infty}{=}1,\, r_{1:t-1})
    = \mathcal{N}\!\left(x_t \;\middle|\; m_{t-1} + S_t\,\Sigma_t,\;\; \Sigma_t\right),
\end{equation}
so the message shifts the predictive mean by $S_t \Sigma_t$ and leaves the variance unchanged. As in Appendix~\ref{app:derivation}, the relevant coefficient is $S_1 = 1/(1-\gamma)$ and the index is
\begin{equation}\label{eq:ef-bonus}
    \lambda_{\mathrm{exp}} = m_{t-1} + \frac{\Sigma_t}{1-\gamma} .
\end{equation}

Next, we evaluate this index under the configurations of Section~\ref{sec:regret} ($K = 4$, $T = 200$, $\gamma = 0.95$, 1000 Monte Carlo runs). Its final discounted regret is $163.81 \pm 2.09$ (mixed), $196.91 \pm 3.94$ ($s$-dominant), and $551.35 \pm 7.70$ ($v$-dominant), against $59.47$, $84.53$, and $323.29$ for the myopic baseline and $39.18$, $74.04$, and $103.30$ for CAUSE. The agent is therefore substantially worse than a policy that never explores, in every regime, and the margin is largest in the $v$-dominant regime, where drift replenishes uncertainty fastest.

Two properties of Eq.~\eqref{eq:ef-bonus} account for this failure, and both are structural. First, stochasticity has no direct effect on the index: by Eq.~\eqref{eq:ef-local}, $s$ enters the message only through a term constant in $x_t$, which cancels on normalization. The bonus therefore depends on the noise parameters only through the total predictive variance, and cannot distinguish uncertainty generated by process noise from uncertainty generated by observation noise. This is precisely the discrimination that Theorems~\ref{thm:mono_s} and~\ref{thm:mono_v} require of the optimal bonus, and in this respect the exponential index reproduces the deficiency of UCB, reached from a different direction. Second, volatility is confined to the same additive constant: it enters the backward message only through the additive constant $E_t$, which cancels in the posterior, so the effective planning horizon is set by $\gamma$ alone. There is no analogue of the damping factor $D$ of Eq.~\eqref{eq:D}, and hence none of the horizon-shortening behavior that Section~\ref{sec:cause} identifies. 

%% file: checklist.tex
\section*{NeurIPS Paper Checklist}

\begin{enumerate}

\item {\bf Claims}
    \item[] Question: Do the main claims made in the abstract and introduction accurately reflect the paper's contributions and scope?
    \item[] Answer: \answerYes{} 
    \item[] Justification: The abstract's claims are supported by theoretical and empirical evidence (Sections 4, 5, and 6), with proofs and derivations in Appendices A and B. 

\item {\bf Limitations}
    \item[] Question: Does the paper discuss the limitations of the work performed by the authors?
    \item[] Answer: \answerYes{} 
    \item[] Justification: The second paragraph of Section 7 discusses limitations.

\item {\bf Theory assumptions and proofs}
    \item[] Question: For each theoretical result, does the paper provide the full set of assumptions and a complete (and correct) proof?
    \item[] Answer: \answerYes{} 
    \item[] Justification: We provide derivation for our theoretical results in Appendix \ref{app:proofs} and \ref{app:derivation}.

    \item {\bf Experimental result reproducibility}
    \item[] Question: Does the paper fully disclose all the information needed to reproduce the main experimental results of the paper to the extent that it affects the main claims and/or conclusions of the paper (regardless of whether the code and data are provided or not)?
    \item[] Answer: \answerYes{} 
    \item[] Justification: We provide all details of our algorithms and baselines in Appendix \ref{app:setup}.

\item {\bf Open access to data and code}
    \item[] Question: Does the paper provide open access to the data and code, with sufficient instructions to faithfully reproduce the main experimental results, as described in supplemental material?
    \item[] Answer: \answerYes{} 
    \item[] Justification: Code is available at https://github.com/piraylab/cause-exploration.

\item {\bf Experimental setting/details}
    \item[] Question: Does the paper specify all the training and test details (e.g., data splits, hyperparameters, how they were chosen, type of optimizer) necessary to understand the results?
    \item[] Answer: \answerYes{} 
    \item[] Justification: Experimental details are provided in Section~\ref{sec:experiments} and Appendix~\ref{app:setup}.

\item {\bf Experiment statistical significance}
    \item[] Question: Does the paper report error bars suitably and correctly defined or other appropriate information about the statistical significance of the experiments?
    \item[] Answer: \answerYes{} 
    \item[] Justification: We report mean and standard error of the mean for all experimental results.

\item {\bf Experiments compute resources}
    \item[] Question: For each experiment, does the paper provide sufficient information on the computer resources (type of compute workers, memory, time of execution) needed to reproduce the experiments?
    \item[] Answer: \answerYes{} 
    \item[] Justification: We discuss compute resources in Appendix \ref{app:compute}.
        
\item {\bf Code of ethics}
    \item[] Question: Does the research conducted in the paper conform, in every respect, with the NeurIPS Code of Ethics \url{https://neurips.cc/public/EthicsGuidelines}?
    \item[] Answer: \answerYes{} 
    \item[] Justification: This paper conforms with the NeurIPS Code of Ethics.

\item {\bf Broader impacts}
    \item[] Question: Does the paper discuss both potential positive societal impacts and negative societal impacts of the work performed?
    \item[] Answer: \answerNA{} 
    \item[] Justification: The work has no direct deployment context with societal impact.
    
\item {\bf Safeguards}
    \item[] Question: Does the paper describe safeguards that have been put in place for responsible release of data or models that have a high risk for misuse (e.g., pre-trained language models, image generators, or scraped datasets)?
    \item[] Answer: \answerNA{} 
    \item[] Justification: This paper poses no such risks.

\item {\bf Licenses for existing assets}
    \item[] Question: Are the creators or original owners of assets (e.g., code, data, models), used in the paper, properly credited and are the license and terms of use explicitly mentioned and properly respected?
    \item[] Answer: \answerYes{} 
    \item[] Justification: Original owners of assets are credited, and license information is mentioned.

\item {\bf New assets}
    \item[] Question: Are new assets introduced in the paper well documented and is the documentation provided alongside the assets?
    \item[] Answer: \answerNA{} 
    \item[] Justification: The paper does not introduce new datasets, pre-trained models, or software libraries as released assets. 

\item {\bf Crowdsourcing and research with human subjects}
    \item[] Question: For crowdsourcing experiments and research with human subjects, does the paper include the full text of instructions given to participants and screenshots, if applicable, as well as details about compensation (if any)? 
    \item[] Answer: \answerNA{} 
    \item[] Justification:  The paper does not involve crowdsourcing nor research with human subjects.

\item {\bf Institutional review board (IRB) approvals or equivalent for research with human subjects}
    \item[] Question: Does the paper describe potential risks incurred by study participants, whether such risks were disclosed to the subjects, and whether Institutional Review Board (IRB) approvals (or an equivalent approval/review based on the requirements of your country or institution) were obtained?
    \item[] Answer: \answerNA{} 
    \item[] Justification: The paper does not involve crowdsourcing nor research with human subjects.

\item {\bf Declaration of LLM usage}
    \item[] Question: Does the paper describe the usage of LLMs if it is an important, original, or non-standard component of the core methods in this research? Note that if the LLM is used only for writing, editing, or formatting purposes and does \emph{not} impact the core methodology, scientific rigor, or originality of the research, declaration is not required.
    \item[] Answer: \answerNA{} 
    \item[] Justification: LLMs have only been used for drafting, editing, coding assistance, and formatting purposes, with one exception noted for transparency: the approximation in Equation 22 (Appendix B) was initially suggested by Opus 4.8, and was subsequently verified by the authors.

\end{enumerate}